\documentclass{article}

\usepackage[final]{neurips_2019}

\usepackage[utf8]{inputenc} \usepackage[T1]{fontenc}    \usepackage{hyperref}       \usepackage{url}            \usepackage{booktabs}       \usepackage{amsfonts}       \usepackage{nicefrac}       \usepackage{microtype}      

\RequirePackage{mathrsfs, amsthm, amsmath, amsfonts, amssymb, mathtools}\usepackage[titletoc,title]{appendix}\usepackage{dsfont}\usepackage{caption} \usepackage{subcaption}
\usepackage{float}\usepackage{graphicx, color}\usepackage{bm}
\usepackage{etoolbox}
\usepackage{natbib}
\patchcmd{\appendices}{\quad}{. }{}{}

\usepackage{pgfplots}\usepackage{booktabs,multirow,array,multicol}

\usepackage[ruled,linesnumbered,vlined,scleft]{algorithm2e} 
\usepackage{adjustbox}

\usepackage{enumitem}

\setlist[itemize]{leftmargin=4.5mm}

\usepackage{accents}
\usepackage{datenumber}
\usepackage{wrapfig}

\newtheorem{lemma}{Lemma}

\newtheorem{proposition}{Proposition}
\newtheorem{remark}{Remark}
{\bf}{\rm}
\newcommand{\bigO}{\mathcal{O}}
\newcommand{\inr}[1]{\langle #1 \rangle}
\newcommand{\norm}[1]{\|#1\|}
\newcommand{\R}{{\mathbb{R}}}

\newcommand*{\dt}[1]{\accentset{\mbox{\large\bfseries .}}{#1}}

\DeclareMathOperator*{\argmin}{\arg\!\min}

\title{Screening Sinkhorn Algorithm for Regularized Optimal Transport}

\begin{document}

\author{Mokhtar Z. Alaya \\
LITIS EA4108\\
University of Rouen Normandy\\
\texttt{mokhtarzahdi.alaya@gmail.com} 
\And
Maxime Bérar\\
LITIS EA4108\\
University of Rouen Normandy\\
\texttt{maxime.berar@univ-rouen.fr} \\
\And
Gilles Gasso \\
LITIS EA4108\\
INSA, University of Rouen Normandy\\
\texttt{gilles.gasso@insa-rouen.fr} 
\And
Alain Rakotomamonjy\\
LITIS EA4108 \\
University of Rouen Normandy\\
and Criteo AI Lab, Criteo Paris \\
\texttt{alain.rakoto@insa-rouen.fr} \\
}

\maketitle

\begin{abstract}
We introduce in this paper a novel strategy for efficiently approximating the Sinkhorn distance between two discrete measures. After identifying neglectable components
of the dual solution of the regularized Sinkhorn problem, we propose to screen those components by directly setting them at that value before entering the Sinkhorn problem. This allows us to solve a smaller Sinkhorn problem while ensuring approximation with provable guarantees.
More formally, the approach is based on a new formulation of \emph{dual of Sinkhorn divergence problem} and on the KKT optimality conditions of this problem, which enable identification of dual components to be screened.
This new analysis leads to the \textsc{Screenkhorn} algorithm.
We illustrate the efficiency of \textsc{Screenkhorn} on complex tasks such as dimensionality reduction and domain adaptation involving regularized optimal transport.
\end{abstract}

\section{Introduction} \label{sec:introduction}

Computing optimal transport (OT) distances between pairs of probability measures or histograms, such as the earth mover's distance~\citep{werman1985,Rubner2000} and Monge-Kantorovich or Wasserstein distance~\citep{villani09optimal}, are currently generating an increasing attraction in different machine learning tasks~\citep{pmlr-v32-solomon14,kusnerb2015,pmlr-v70-arjovsky17a,ho2017}, statistics~\citep{frogner2015nips,panaretos2016,ebert2017ConstructionON,bigot2017,flamary2018WDA}, and computer vision~\citep{bonnel2011,Rubner2000,solomon2015}, among other applications~\citep{klouri17,peyre2019COTnowpublisher}.
In many of these problems, OT exploits the geometric features of the objects at hand in the underlying spaces to be leveraged in comparing probability measures.
This effectively leads to improved performance of methods that are oblivious to the geometry, for example the chi-squared distances or the Kullback-Leibler divergence.
Unfortunately, this advantage comes at the price of an enormous computational cost of solving the OT problem, that can be prohibitive in large scale applications.
For instance, the OT between two histograms with supports of equal size $n$ can be formulated as a linear programming problem that requires generally super $\bigO(n^{2.5})$~\citep{leeSidford2013PathFI} arithmetic operations, which is problematic when $n$ becomes larger.

A remedy to the heavy computation burden of OT lies in a prevalent approach referred to as regularized OT~\citep{cuturinips13} and operates by adding an entropic regularization penalty to the original problem.  
Such a regularization guarantees a unique solution, since the objective function is strongly convex, and a greater computational stability.
More importantly, this regularized OT can be solved efficiently with celebrated matrix scaling algorithms, such as Sinkhorn's fixed point iteration method~\citep{sinkhorn1967,knight2008,kalantari2008}. 

Several works have considered further improvements in the resolution of this regularized OT problem.
A greedy version of Sinkhorn algorithm, called Greenkhorn~\cite{altschulernips17}, allows to select and update columns and rows that most violate the polytope constraints.                   
Another approach based on low-rank approximation of the cost matrix using the Nystr\"om method induces the Nys-Sink algorithm~\citep{altschuler2018Nystrom}. 
Other classical optimization algorithms have been considered for approximating the OT, for instance accelerated gradient descent~\citep{xie2018proxpointOT,dvurechensky18aICML,lin2019}, quasi-Newton methods~\citep{blondel2018ICML,cuturi2016SIAM} and stochastic gradient descent~\citep{genevay2016stochOT,khalilabid2018}. 

{In this paper, we propose a novel technique for accelerating the Sinkhorn algorithm when computing regularized OT distance between discrete measures. Our idea
is strongly related to a screening strategy when solving a \emph{Lasso}
problem in sparse supervised learning \citep{Ghaoui2010SafeFE}. Based on the fact
that a  transport plan resulting from an OT problem is sparse or presents a large
number of neglectable values \citep{blondel2018ICML}, our objective is to identify the  dual variables of an approximate Sinkhorn problem, that are smaller than a predefined threshold, and thus that can be safely removed before optimization while not altering too much the solution of the problem.  
Within this global context, our contributions are the following:
\begin{itemize}
	  \setlength\itemsep{-0.1cm}
	
	\item From a methodological point of view, we propose a new formulation of the dual of the Sinkhorn divergence problem by imposing variables to be larger than a threshold.
	This formulation allows us to introduce sufficient conditions, computable beforehand, for a variable to  strictly satisfy its constraint, leading then to
	a ``screened'' version of the dual of Sinkhorn divergence. 
	\item We provide some theoretical analysis of the solution of the ``screened''  Sinkhorn divergence, showing that its objective value and the marginal constraint satisfaction are properly controlled 	as the number of screened variables decreases.
	\item From an algorithmic standpoint, we use a constrained L-BFGS-B algorithm \citep{nocedal1980,byrd1995L-BFGS-B} but provide a careful analysis of the lower and upper bounds of the dual	variables, resulting in a well-posed and efficient algorithm denoted as \textsc{Screenkhorn}.
	\item Our empirical analysis depicts how the approach behaves in a simple Sinkhorn divergence computation context. When considered in  complex machine learning
	pipelines, we show that \textsc{Screenkhorn} can lead to strong gain in efficiency
	while not compromising on accuracy.
\end{itemize}}

The remainder of the paper is organized as follow. In Section~\ref{sec:regularized_discrete_ot} we briefly review the basic setup of regularized discrete OT. 
Section~\ref{sec:screened_dual_of_sinkhorn_divergence} contains our main contribution, that is, the \textsc{Screenkhorn} algorithm. 
Section~\ref{sec:analysis_of_marginal_violations} is devoted to theoretical guarantees for marginal violations of \textsc{Screenkhorn}. 
In Section~\ref{sec:numerical_experiments} we present numerical results for the proposed algorithm, compared with the state-of-art Sinkhorn algorithm as implemented in~\cite{flamary2017pot}. 
The proofs of theoretical results are postponed to the supplementary material {as well as additional empirical results}.

\emph{Notation.} For any positive matrix $T \in \R^{n\times m}$, we define its entropy as $H(T) = -\sum_{i,j} T_{ij} \log(T_{ij}).$
Let $r(T) = T\mathbf 1_m \in \R^n$ and $c(T) = T^\top\mathbf 1_n \in \R^m$ denote the rows and columns sums of $T$ respectively. The coordinates $r_i(T)$ and $c_j(T)$ denote the $i$-th row sum and the $j$-th column sum of $T$, respectively.
The scalar product between two matrices denotes the usual inner product, that is $\inr{T, W} = \text{tr}(T^\top W) = \sum_{i,j}T_{ij}W_{ij},$ where $T^\top$ is the transpose of $T$. 
We write $\mathbf{1}$ (resp. $\mathbf{0}$) the vector having all coordinates equal to one (resp. zero).
$\Delta(w)$ denotes the diag operator, such that if $w \in \R^n$, then $\Delta(w) = \text{diag}(w_1, \ldots, w_n)\in \R^{n\times n}$.
For a set of indices $L=\{i_1, \ldots, i_k\} \subseteq \{1, \ldots, n\}$ satisfying $i_1 < \cdots <i_k,$ we denote the complementary set of $L$ by $L^\complement = \{1, \ldots, n\} \backslash L$. We also denote $|L|$ the cardinality of $L$.
Given a vector $w \in \R^n$, we denote $w_L= (w_{i_1}, \ldots, w_{i_k})^\top \in \R^k$ and its complementary $w_{L^\complement} \in \R^{n- k}$.  The notation is similar for matrices; given another subset of indices $S = \{j_1, \ldots, j_l\} \subseteq \{1, \ldots, m\}$ with $j_1 < \cdots <j_l,$ and a matrix $T\in \R^{n\times m}$, we use $T_{(L,S)}$, to denote the submatrix of $T$, namely the rows and columns of $T_{(L,S)}$ are indexed by $L$ and $S$ respectively.
When applied to matrices and vectors,  $\odot$ and $\oslash$ (Hadamard product and division) and exponential notations refer to elementwise operators.
Given two real numbers $a$ and $b$, we write $a\vee b = \max(a,b)$ and $a\wedge b = \min(a,b).$


\section{Regularized discrete OT} \label{sec:regularized_discrete_ot}

We briefly expose in this section the setup of OT between two discrete measures. We then consider the case when those distributions are only available through a finite number of samples, that is $\mu = \sum_{i=1}^n \mu_i \delta_{x_i} \in \Sigma_n$ and $\nu = \sum_{j=1}^m \nu_i \delta_{y_j} \in \Sigma_m$, where $\Sigma_n$ is the probability simplex with $n$ bins, namely the set of probability vectors in $\R_+^n$, i.e., $\Sigma_n = \{w \in \R_+^n: \sum_{i=1}^n w_i = 1\}.$
We denote their probabilistic couplings set as $\Pi(\mu, \nu) = \{P \in \R_+^{n\times m}, P\mathbf{1}_m = \mu, P^\top \mathbf{1}_n = \nu\}.$ 
\paragraph{Sinkhorn divergence.}

Computing OT distance between the two discrete measures $\mu$ and $\nu$  amounts to solving a linear problem~\citep{kantorovich1942} given by
\begin{equation*}
  \label{monge-kantorovich}
  \mathcal{S}(\mu, \nu) =  \min_{P\in \Pi(\mu, \nu)} \inr{C, P},
\end{equation*}
where $P= (P_{ij}) \in \R^{n\times m}$ is called the transportation plan, namely each entry $P_{ij}$ represents the fraction of mass moving from $x_i$ to $y_j$, and $C= (C_{ij}) \in \R^{n\times m}$ is a cost matrix comprised of nonnegative elements and related to the energy needed to move a probability mass from $x_i$ to $y_j$. 
The entropic regularization of OT distances~\citep{cuturinips13} relies on the addition of a penalty term as follows:
\begin{equation}
\label{sinkhorn-primal}
  \mathcal{S}_\eta(\mu, \nu) =  \min_{P\in \Pi(\mu, \nu)} \{\inr{C, P} - \eta H(P)\},
\end{equation}
where $\eta > 0$ is a regularization parameter. We refer to $\mathcal{S}_\eta(\mu, \nu) $ as the \emph{Sinkhorn divergence}~\citep{cuturinips13}.

\paragraph{Dual of Sinkhorn divergence.}

Below we provide the derivation of the dual problem for the regularized OT problem~\eqref{sinkhorn-primal}. Towards this end, we begin with writing its Lagrangian dual function:
\begin{equation*}
  \mathscr{L}(P,w, z) = \inr{C,P} + \eta \inr{\log P, P} + \inr{w, P\mathbf{1}_m - \mu} + \inr{z,P^\top \mathbf{1}_n - \nu}.
\end{equation*}
The dual of Sinkhorn divergence can be derived by solving $\min_{P \in \R_+^{n\times m}}\mathscr{L}(P,w, z)$. It is easy to check that objective function $P\mapsto \mathscr{L}(P,w, z)$ is strongly convex and differentiable. Hence, one can solve the latter minimum by setting $\nabla_P \mathscr{L}(P,w, z)$ to $\mathbf{0}_{n\times m}$. Therefore, we get $  P^\star_{ij} = \exp\big(- \frac{1}{\eta} (w_i + z_j + C_{ij}) - 1\big), 
$ for all $i=1, \ldots, n$ and $j=1, \ldots, m$. Plugging this solution,  and setting the change of variables $u = -w/\eta - 1/2$ and $v = - z/\eta - 1/2$, the dual problem is given by
\begin{equation}
\label{sinkhorn-dual}
\min_{u \in \R^n, v\in\R^m}\big\{\Psi(u,v):= \mathbf{1}_n^\top B(u,v)\mathbf{1}_m - \inr{u, \mu} - \inr{v, \nu} \big\},
\end{equation}
where $B(u,v) := \Delta(e^{u}) K \Delta(e^{v})$ and $K := e^{-C/\eta}$ stands for the Gibbs kernel associated to the cost matrix $C$. 
We refer to problem~\eqref{sinkhorn-dual} as the \emph{dual of Sinkhorn divergence}. Then, the optimal solution $P^\star$ of the primal problem~\eqref{sinkhorn-primal} takes the form $P^\star = \Delta(e^{u^\star}) K \Delta(e^{v^\star})$
where the couple $(u^\star, v^\star)$ satisfies:
\begin{align*}
\label{sinkhorn-dual}
  (u^\star, v^\star) &= \argmin_{u \in \R^{n}, v\in \R^m} \{\Psi(u,v)\}.
\end{align*}
Note that the matrices $\Delta(e^{u^\star})$ and $\Delta(e^{v^\star})$ are unique up to a constant factor~\citep{sinkhorn1967}. Moreover, $P^\star$ can be solved efficiently by iterative Bregman projections~\citep{benamou2015IterativeBP} referred to as Sinkhorn iterations, and the method is referred to as \textsc{Sinkhorn} algorithm which, recently, has been proven to achieve a near-$\bigO(n^2)$ complexity~\citep{altschulernips17}.


\section{Screened dual of Sinkhorn divergence} \label{sec:screened_dual_of_sinkhorn_divergence}

\begin{wrapfigure}{o}{0.3\textwidth}
\vspace{-12pt}
\centering
\includegraphics[width=0.3\textwidth]{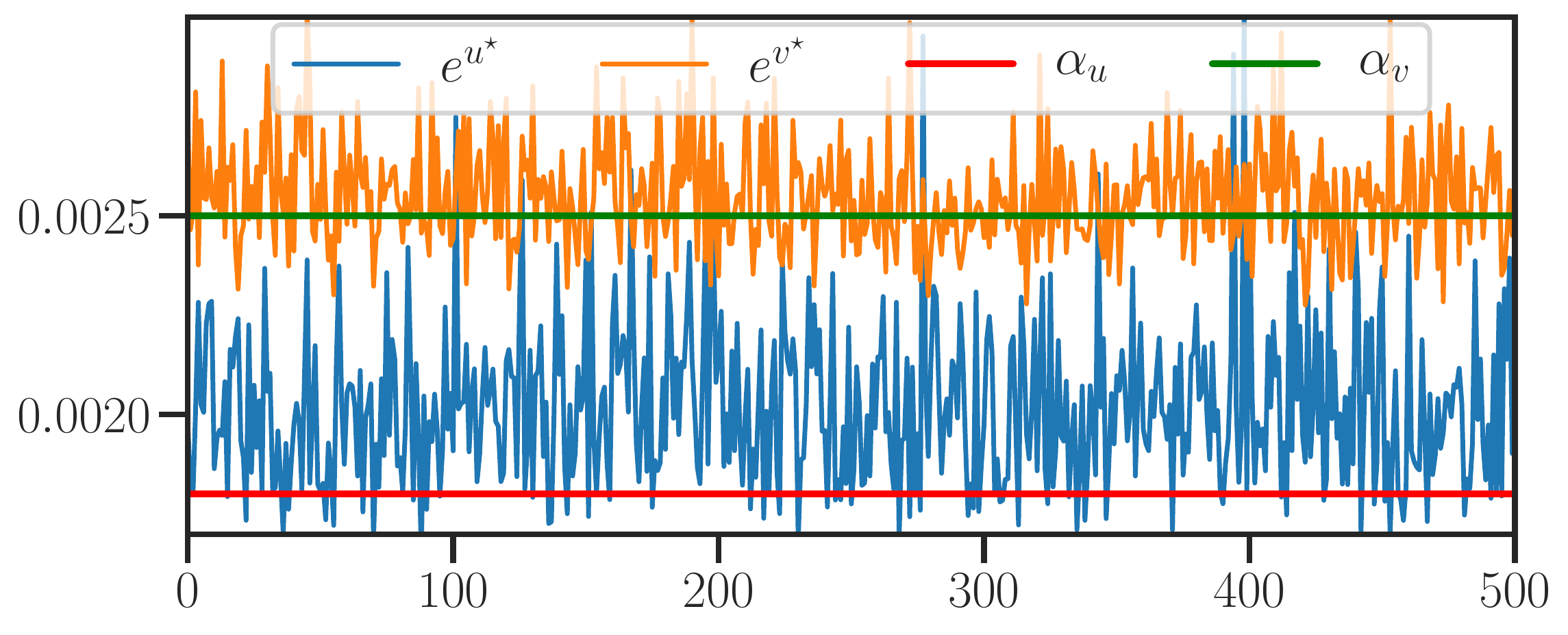}
\caption{Plots of $(e^{u^\star}, e^{v^\star})$ with $(u^\star, v^\star)$ is the pair solution of dual of Sinkhorn divergence~\eqref{sinkhorn-dual} and the thresholds $\alpha_u, \alpha_v$.}
\label{fig:motivations}
\vspace{-12pt}
\end{wrapfigure}

\paragraph{Motivation.} 

The key idea of our approach is motivated by the so-called \emph{static screening test}~\citep{Ghaoui2010SafeFE} in supervised learning, which is a method able to {safely} identify inactive features, i.e., features that have zero components in the solution vector. 
Then, these inactive features can be removed from the optimization problem to reduce its scale.
Before diving into detailed algorithmic analysis, let us present a brief illustration of how we adapt static screening test to the dual of Sinkhorn divergence.
Towards this end, we define the convex set $\mathcal{C}^r_{\alpha} \subseteq \R^r$, for $r\in \mathbb N$ and $\alpha >0$, by $\mathcal{C}^r_{\alpha} = \{w\in \R^{r}:  e^{w_i} \geq \alpha\}$.
In Figure~\ref{fig:motivations}, we plot $(e^{u^\star}, e^{v^\star})$ where $(u^\star, v^\star)$ is the pair solution of the dual of Sinkhorn divergence~(\ref{sinkhorn-dual}) in the particular case of: $n=m=500, \eta=1, \mu = \nu = \frac 1n \mathbf 1_n, x_i \sim\mathcal{N}((0,0)^\top, \begin{psmallmatrix}1 & 0\\0 & 1\end{psmallmatrix}), y_j \sim\mathcal{N}((3,3)^\top, \begin{psmallmatrix}1 &-0.8 \\ -0.8 &1 \end{psmallmatrix})$ and the cost matrix $C$ corresponds to the pairwise euclidean distance, i.e., $C_{ij} = \norm{x_i - y_j}_2$. 
We also plot two lines corresponding to $e^{u^\star} \equiv \alpha_u$ and $e^{v^\star} \equiv \alpha_v$ for some $\alpha_u>0$ and $\alpha_v >0$, choosing randomly and playing the role of thresholds to select indices to be discarded. {If we are able to identify these indices before solving the problem, they can be fixed at the thresholds and removed then from the optimization procedure yielding an approximate solution.}

\paragraph{Static screening test.} Based on this idea, we define a so-called \emph{approximate dual of Sinkhorn divergence} 
\begin{equation} 
\label{screen-sinkhorn}
\min_{u \in \mathcal{C}^n_{\frac \varepsilon \kappa}, v\in \mathcal{C}^m_{\varepsilon\kappa}} \big\{\Psi_{\kappa}(u,v):= \mathbf{1}_n^\top B(u,v)\mathbf{1}_m - \inr{\kappa u, \mu} - \inr{\frac v\kappa, \nu} \big\},
\end{equation}
which is simply a dual of Sinkhorn divergence with lower-bounded variables, where the bounds
are $\alpha_u = \varepsilon \kappa^{-1}$ and $\alpha_v = \varepsilon \kappa$ with $\varepsilon > 0$ and $\kappa > 0$ being fixed numeric constants which values will be
clear later. 
{The new formulation~\eqref{screen-sinkhorn} has the form of $(\kappa\mu, \nu / \kappa)$-scaling problem under constraints on the variables $u$ and $v$.  Those constraints make the problem  significantly different from the standard scaling-problems~\citep{KALANTARI199687}. 
We further emphasize that $\kappa$ plays a key role in our screening strategy.  Indeed, without $\kappa$, $e^u$ and $e^v$ can have inversely related scale  that may lead in,  for instance $e^u$ being too large and $e^v$ being too small, situation in which the screening test would apply only to coefficients of $e^u$ or $e^v$ and not for both of them.
Moreover, it is clear that the approximate dual of Sinkhorn divergence coincides with the dual of Sinkhorn divergence~\eqref{sinkhorn-dual} when $\varepsilon=0$ and $\kappa=1$.
 {Intuitively, our hope is to gain efficiency in solving
problem (\ref{screen-sinkhorn}) compared to the original one in Equation \eqref{sinkhorn-dual} by avoiding optimization of variables smaller than the threshold and by identifying
those that make the constraints active. } More formally, 
 the core of the static screening test aims at locating two subsets of indices $(I, J)$ in $\{1, \ldots, n\}\times\{1, \ldots, m\}$ satisfying: $e^{u_i} > \alpha_u, \text{ and } e^{v_j} > \alpha_v, \text{ for all } (i,j) \in I \times J$ and 
$e^{u_{i'}} = \alpha_u, \text{ and } e^{v_{j'}} = \alpha_v, \text{ for all } (i',j') \in I^\complement \times J^\complement$, namely $(u,v) \in \mathcal{C}^n_{\alpha_u}\times \mathcal{C}^m_{\alpha_v}$. {The following key result states sufficient conditions for identifying variables in $I^\complement$ and $J^\complement$.}

\begin{lemma}
\label{lemma_actives_sets}
Let $(u^{*}, v^{*})$ be an optimal solution of problem~\eqref{screen-sinkhorn}. 
Define
\begin{equation}
\label{I_epsilon_kappa_J_epsilon_kappa}
I_{\varepsilon,\kappa} = \big\{i=1, \ldots, n: \mu_i \geq \frac {\varepsilon^2} \kappa^{} r_i(K)\big\}, J_{\varepsilon,\kappa} = \big\{j=1, \ldots, m: \nu_j \geq \kappa{\varepsilon^2}{} c_j(K)\big\}
\end{equation}
Then one has $e^{u^{*}_i} = \varepsilon\kappa^{-1}$ and $e^{v^{*}_j} = \varepsilon\kappa$ for all $i \in I^\complement_{\varepsilon,\kappa} $ and $j\in J^\complement_{\varepsilon,\kappa} .$
\end{lemma}

Proof of Lemma~\ref{lemma_actives_sets} is postponed to the supplementary material. It is worth to note that first order optimality conditions applied to $(u^{*}, v^{*})$ ensure that if $e^{u^{*}_i} > \varepsilon\kappa^{-1}$ then $e^{u^{*}_i} (Ke^{v^{*}})_i =  \kappa\mu_i$ and if $e^{v^{*}_j} > \varepsilon\kappa$ then $e^{v^{*}_j} (K^\top e^{u^{*}})_j =  \kappa^{-1}\nu_j$, that correspond to the Sinkhorn marginal conditions~\citep{peyre2019COTnowpublisher} up to the scaling factor $\kappa$. 

\paragraph{Screening with a fixed number budget of points.}
The approximate dual of Sinkhorn divergence is defined with respect to $\varepsilon$ and $\kappa$. {As those parameters are
	difficult to interpret, we exhibit their relations with  a {fixed number budget of points} from the supports of $\mu$ and $\nu$}.
In the sequel, we denote by $n_b \in\{1, \ldots, n\}$ and $m_b\in\{1, \ldots, m\}$ the number of points {that are going to be optimized in problem~\eqref{screen-sinkhorn}, \emph{i.e}, the points we cannot guarantee
that  $e^{u^{*}_i} = \varepsilon\kappa^{-1}$ and $e^{v^{*}_j} = \varepsilon\kappa$ }. 

Let us define $\xi \in \R^n$ and $\zeta \in \R^m$ to be the ordered decreasing vectors of $\mu \oslash r(K)$ and $\nu \oslash c(K)$ respectively, that is $\xi_1 \geq \xi_2 \geq \cdots \geq \xi_n$ and $\zeta_1 \geq \zeta_2 \geq \cdots \geq \zeta_m$.
To keep only $n_b$-budget and $m_b$-budget of points, the parameters $\kappa$ and $\varepsilon$ satisfy ${\varepsilon^2}\kappa^{-1} = \xi_{n_b}$ and $\varepsilon^2\kappa = \zeta_{m_b}$. Hence 
\begin{equation}
\label{epsilon_kappa}
 \varepsilon = (\xi_{n_b}\zeta_{m_b})^{1/4} \text{ and } \kappa = \sqrt{\frac{\zeta_{m_b}}{\xi_{n_b}}}.
\end{equation}
This guarantees that $|I_{\varepsilon, \kappa}| = n_b$ and $|J_{\varepsilon, \kappa}| = m_b$ by construction. In addition, when $(n_b,m_b)$ tends to the full number budget of points $(n,m)$, the objective in problem \eqref{screen-sinkhorn} converges to the objective of dual of Sinkhorn divergence~\eqref{sinkhorn-dual}.

{We are now in position to formulate the
optimization problem related to the screened dual of Sinkhorn.} Indeed, using the above analyses, any solution $(u^*, v^*)$ of problem~\eqref{screen-sinkhorn} satisfies $e^{u^*_i} \geq \varepsilon\kappa^{-1}$ and $e^{v^*_j} \geq \varepsilon\kappa$ for all $(i,j) \in (I_{\varepsilon,\kappa}\times J_{\varepsilon,\kappa}),$ and $e^{u^*_i} = \varepsilon\kappa^{-1}$ and $e^{v^*_j} = \varepsilon\kappa$ for all $(i,j) \in (I^\complement_{\varepsilon,\kappa}\times J^\complement_{\varepsilon,\kappa})$.
{Hence, we can restrict the problem \eqref{screen-sinkhorn} to 
variables in $I_{\varepsilon,\kappa}$ and $J_{\varepsilon,\kappa}$}. This boils down
to restricting the constraints feasibility $\mathcal{C}^n_{\frac \varepsilon \kappa} \cap \mathcal{C}^m_{\varepsilon\kappa}$ to the screened domain defined by $\mathcal{U}_{\text{sc}} \cap \mathcal{V}_{\text{sc}}$, 
\begin{equation*}
\mathcal{U}_{\text{sc}} = \{u \in \R^{n_b}: e^{u_{I_{\varepsilon,\kappa}}} \succeq \frac \varepsilon\kappa\mathbf 1_{n_b}\} \text{ and } \mathcal{V}_{\text{sc}} =\{v\in\R^{m_b}: e^{v_{J_{\varepsilon,\kappa}}} \succeq \varepsilon\kappa \mathbf{1}_{m_b}\}\end{equation*}
where the vector comparison $\succeq$ has to be understood elementwise. {And,
	  by replacing in Equation \eqref{screen-sinkhorn}, the variables belonging to $(I^\complement_{\varepsilon,\kappa}\times J^\complement_{\varepsilon,\kappa})$ by $\varepsilon\kappa^{-1}$ and
$\varepsilon\kappa$}, we derive the \emph{screened dual of Sinkhorn divergence problem} as
\begin{align}
\label{screen-sinkhorn_second_def}
\min_{u \in \mathcal{U}_{\text{sc}}, v \in \mathcal{V}_{\text{sc}}}\{\Psi_{\varepsilon, \kappa}(u,v)\}
\end{align}
where 
\begin{align*} 
\Psi_{\varepsilon,\kappa}(u, v) &= (e^{u_{I_{\varepsilon,\kappa}}})^\top K_{(I_{\varepsilon,\kappa}, J_{\varepsilon,\kappa})} e^{v_{J_{\varepsilon,\kappa}}} + 
\varepsilon \kappa (e^{u_{I_{\varepsilon,\kappa}}})^\top K_{(I_{\varepsilon,\kappa}, J^\complement_{\varepsilon,\kappa})}\mathbf 1_{m_b} + \varepsilon \kappa^{-1} \mathbf 1_{n_b}^\top K_{(I^\complement_{\varepsilon,\kappa}, J_{\varepsilon,\kappa})}e^{v_{J_{\varepsilon,\kappa}}}\\
&\qquad - \kappa \mu_{I_{\varepsilon,\kappa}}^\top u_{I_{\varepsilon,\kappa}} - \kappa^{-1} \nu_{J_{\varepsilon,\kappa}}^\top v_{J_{\varepsilon,\kappa}} + \Xi
\end{align*}
with $\Xi = \varepsilon^2 \sum_{i \in I^\complement_{\varepsilon,\kappa}, j \in J^\complement_{\varepsilon,\kappa}} K_{ij} -\kappa \log(\varepsilon\kappa^{-1})\sum_{i \in I^\complement_{\varepsilon,\kappa}}\mu_i - \kappa^{-1} \log(\varepsilon\kappa)\sum_{j\in J^\complement_{\varepsilon,\kappa}} \nu_j$.

 The above problem uses only the restricted parts $K_{(I_{\varepsilon,\kappa}, J_{\varepsilon,\kappa})},$ $K_{(I_{\varepsilon,\kappa}, J^\complement_{\varepsilon,\kappa})},$ and $K_{(I^\complement_{\varepsilon,\kappa}, J_{\varepsilon,\kappa})}$ of the Gibbs kernel $K$ for calculating the objective function $\Psi_{\varepsilon, \kappa}$. Hence, a gradient descent scheme will also need only those rows/columns of $K$. This is in contrast to Sinkhorn algorithm which performs alternating updates of all rows and columns of $K$. In summary, \textsc{Screenkhorn} consists of two steps: the first one is a screening pre-processing providing the active sets $I_{\varepsilon,\kappa}$, $J_{\varepsilon,\kappa}$. 
The second one consists in solving Equation \eqref{screen-sinkhorn_second_def}
using a constrained L-BFGS-B \citep{byrd1995L-BFGS-B} for the stacked variable $\theta=(u_{I_{\varepsilon,\kappa}},v_{J_{\varepsilon,\kappa}}).$ 
Pseudocode of our proposed algorithm is shown in Algorithm~\ref{screenkhorn}. {
	Note that in practice, we initialize the L-BFGS-B algorithm based on the output of a  method, called \textsc{Restricted Sinkhorn} (see Algorithm~\ref{restricted_sinkhorn} in the supplementary), which is a Sinkhorn-like algorithm applied to the active dual variables $\theta=(u_{I_{\varepsilon,\kappa}},v_{J_{\varepsilon,\kappa}}).$ While simple and efficient, the solution of this 
	\textsc{Restricted Sinkhorn} algorithm does not satisfy the lower bound constraints of Problem \eqref{screen-sinkhorn_second_def} but provide a good candidate solution.
}
Also note that L-BFGS-B handles box constraints on variables, but it becomes more efficient when these box bounds are carefully determined for problem~\eqref{screen-sinkhorn_second_def}. 
The following proposition (proof in supplementary material) expresses these bounds that are pre-calculated in the initialization step of \textsc{Screenkhorn}.
\begin{proposition}
\label{prop:bounds_of_usc_and_vsc}
Let $(u^{\text{sc}}, v^{\text{sc}})$ be an optimal pair solution of problem~\eqref{screen-sinkhorn_second_def} and $K_{\min} = \min\limits_{i\in I_{\varepsilon,\kappa},j \in J_{\varepsilon,\kappa}}K_{ij}$. Then,
one has
\begin{equation}
\label{bound_on_u}
\frac \varepsilon\kappa \vee \frac{\min_{i \in I_{\varepsilon,\kappa}}\mu_i}{\varepsilon (m- m_b) + \frac{\max_{j\in J_{\varepsilon,\kappa}} \nu_j}{n\varepsilon\kappa K_{\min}} m_b} \leq e^{u^{\text{sc}}_i} \leq \frac{\max_{i \in I_{\varepsilon,\kappa}} \mu_i}{m\varepsilon K_{\min}},
\end{equation}
and
\begin{equation}
\label{bound_on_v}
\varepsilon\kappa \vee \frac{\min_{j \in J_{\varepsilon,\kappa}}\nu_j}{\varepsilon(n- n_b) + \frac{\kappa\max_{i\in I_{\varepsilon,\kappa}} \mu_i}{m\varepsilon K_{\min} } n_b} \leq e^{v^{\text{sc}}_j} \leq  \frac{\max_{j \in J_{\varepsilon,\kappa}} \nu_j}{n\varepsilon K_{\min} }
\end{equation}
for all $i\in I_{\varepsilon,\kappa}$ and $j\in J_{\varepsilon,\kappa}$.
\end{proposition}

\LinesNotNumbered
\begin{algorithm}[tbp]
\SetNlSty{textbf}{}{.}
\DontPrintSemicolon
\caption{\textsc{Screenkhorn}$(C,\eta,\mu,\nu,n_b,m_b)$}
\label{screenkhorn}

\textbf{Step 1:} \textcolor{black}{{Screening pre-processing}}\vspace{.1cm}\\

\nl   $\xi \gets \texttt{sort}(\mu \oslash r(K)),$ $\zeta \gets \texttt{sort}(\nu \oslash c(K));$ //(decreasing order)\\
\nl   $\varepsilon \gets (\xi_{n_b}\zeta_{m_b})^{1/4}, \text{  } \kappa \gets \sqrt{{\zeta_{m_b}}/{\xi_{n_b}}}$;\\
\nl   $I_{\varepsilon,\kappa} \gets \{i=1, \ldots, n: \mu_i \geq {\varepsilon^2} \kappa^{-1} r_i(K)\}, J_{\varepsilon,\kappa} \gets \{j=1, \ldots, m: \nu_j \geq \varepsilon^2\kappa c_j(K)\};$\\ 
\nl   $\underline{\mu} \gets \min_{i \in I_{\varepsilon,\kappa}} \mu_i, \bar{\mu} \gets \max_{i \in I_{\varepsilon,\kappa}} \mu_i, \underline{\nu} \gets \min_{j \in J_{\varepsilon,\kappa}} \nu_i, \bar{\nu} \gets \max_{j \in J_{\varepsilon,\kappa}} \nu_i$; \\
\nl   $\underline{u} \gets \log\big(\frac \varepsilon\kappa \vee \frac{\underline{\mu}}{\varepsilon (m-m_b) + \varepsilon \vee \frac{\bar{\nu}}{n\varepsilon\kappa K_{\min}} m_b}\big), \bar{u} \gets  \log\big(\frac{\bar{\mu}}{m\varepsilon K_{\min}}\big);$\\
\nl   $\underline{v} \gets \log\big(\varepsilon\kappa \vee \frac{\underline{\nu}}{\varepsilon(n-n_b) + \varepsilon \vee \frac{\kappa\bar{\mu}}{m\varepsilon K_{\min}} n_b}\big), \bar{v} \gets \log\big(\frac{\bar{\nu}}{n\varepsilon K_{\min}}\big);$\\
\nl   $ \bar{\theta} \gets \texttt{stack}(\bar{u}\mathbf 1_{n_b}, \bar{v}\mathbf 1_{m_b}),$ $ \underline{\theta} \gets \texttt{stack}(\underline{u}\mathbf 1_{n_b}, \underline{v}\mathbf 1_{m_b}) ;$\\

\vspace{.2cm}
\noindent \textbf{Step 2:} \textcolor{black}{{L-BFGS-B solver on the screened variables}}\vspace{.1cm}\\
{
\nl  $u^{(0)}\gets \log(\varepsilon\kappa^{-1}) \mathbf 1_{n_b},$ $v^{(0)} \gets \log(\varepsilon\kappa) \mathbf 1_{m_b}$;\\
\nl $\hat u, \hat v \gets$ \textsc{Restricted~Sinkhorn}($u^{(0)},v^{(0)}$), $\theta^{(0)} \gets \texttt{stack}(\hat u, \hat v);$\\
}
\nl   $\theta \gets \text{L-BFGS-B}(\theta^{(0)}, \underline{\theta}, \bar{\theta});$\\
\nl   $\theta_u \gets (\theta_1, \ldots, \theta_{n_b})^\top, \theta_v \gets(\theta_{n_b+1}, \ldots, \theta_{n_b+m_b})^\top;$\\
\nl   {$u^{sc}_i \gets (\theta_u)_i$ if $i \in I_{\varepsilon,\kappa}$ and $u_i \gets \log(\varepsilon\kappa^{-1})$ if $i \in I^\complement_{\varepsilon,\kappa};$}\\
\nl   {$v^{sc}_j \gets (\theta_v)_j$ if $j \in J_{\varepsilon,\kappa}$ and $v_j \gets \log(\varepsilon\kappa)$ if $j \in J^\complement_{\varepsilon,\kappa};$}\\
\nl   \Return{$B(u^{\text{sc}},v^{\text{sc}})$.}
\end{algorithm}


\section{Theoretical analysis and guarantees} \label{sec:analysis_of_marginal_violations}

This section is devoted to establishing theoretical guarantees for \textsc{Screenkhorn} algorithm. We first define the screened marginals $\mu^{\text{sc}} = B(u^{\text{sc}}, v^{\text{sc}}) \mathbf 1_m$ and $\nu^{\text{sc}} = B(u^{\text{sc}}, v^{\text{sc}})^\top \mathbf 1_n.$ 
Our first theoretical result, Proposition~\ref{proposition_error_in_marginals}, gives an upper bound of the screened marginal violations with respect to $\ell_1$-norm.

\begin{proposition}
\label{proposition_error_in_marginals}
Let $(u^{\text{sc}}, v^{\text{sc}})$ be an optimal pair solution of problem~\eqref{screen-sinkhorn_second_def}.
Then one has 
{\small{
\begin{align}
\label{marginal-error-mu}
{\norm{{\mu} -{\mu}^{\text{sc}}}^2_1} = \bigO\Big(n_bc_\kappa + (n- n_b) \Big(\frac{\norm{C}_\infty}{\eta} + \frac{m_b}{\sqrt{nmc_{\mu\nu}}K_{\min}^{3/2}} &+ \frac{m-m_b}{\sqrt{nm}K_{\min}}
 + \log\Big(\frac{\sqrt{nm}}{m_bc_{\mu\nu}^{5/2}} 
\Big)\Big)\Big)
\end{align}
}}
and 
{\small{
\begin{align}
\label{marginal-error-nu}
{\norm{{\nu} -{\nu}^{\text{sc}}}^2_1} = \bigO\Big(m_bc_{\frac1\kappa} + (m- m_b) \Big(\frac{\norm{C}_\infty}{\eta} + \frac{n_b}{\sqrt{nmc_{\mu\nu}}K_{\min}^{3/2}} &+ \frac{n-n_b}{\sqrt{nm}K_{\min}}
 + \log\Big(\frac{\sqrt{nm}}{n_b c_{\mu\nu}^{5/2}}
\Big)\Big)\Big),
\end{align}
}}
where $c_z = z - \log z - 1$ for $z>0$ and $c_{\mu\nu} = \underline{\mu}\wedge \underline{\nu}$ with $\underline{\mu} = \min_{i\in I_{\varepsilon,\kappa}}\mu_i$ and $\underline{\nu} = \min_{j\in J_{\varepsilon,\kappa}}\nu_j$.

\end{proposition}
Proof of Proposition~\ref{proposition_error_in_marginals} is presented in supplementary material and it is based on first order optimality conditions for problem~\eqref{screen-sinkhorn_second_def} and on a generalization of Pinsker inequality (see Lemma~\ref{lem:pinsker} in supplementary).

Our second theoretical result, Proposition~\ref{prop:objective-error}, is an upper bound of the difference between objective values of \textsc{Screenkhorn} and dual of Sinkhorn divergence~\eqref{sinkhorn-dual}. 
\begin{proposition}
\label{prop:objective-error}
Let $(u^{\text{sc}}, v^{\text{sc}})$ be an optimal pair solution of problem~\eqref{screen-sinkhorn_second_def} and $(u^\star, v^\star)$ is the pair solution of dual of Sinkhorn divergence~\eqref{sinkhorn-dual}. Then we have 
\begin{align*}
\Psi_{\varepsilon, \kappa}(u^{\text{sc}} ,v^{\text{sc}}) -\Psi(u^\star, v^\star)
= \bigO\big(R(\norm{\mu - \mu^{\text{sc}}}_1 + \norm{\nu - \nu^{\text{sc}}}_1 + \omega_{\kappa})\big).
\end{align*}
where $R = \frac{\norm{C}_\infty}{\eta} + \log\big(\frac{(n\vee m)^2}{nmc_{\mu\nu}^{7/2}}\big)$ and $\omega_{\kappa} = |1- \kappa|\norm{\mu^{\text{sc}}}_1 + |1 - \kappa^{-1}|\norm{\nu^{\text{sc}}}_1 + |1- \kappa| + |1 - \kappa^{-1}|$.
\end{proposition}
Proof of Proposition~\ref{prop:objective-error} is exposed in the supplementary material.
Comparing to some other analysis results of this quantity, see for instance Lemma 2 in~\cite{dvurechensky18aICML} and Lemma 3.1 in~\cite{lin2019}, our bound involves an additional term $\omega_{\kappa}$ (with $\omega_1 =0)$. {To better characterize $\omega_\kappa$, a control of the $\ell_1$-norms of the screened marginals $\mu^{\text{sc}}$ and $\nu^{\text{sc}}$ are given in Lemma 2 in the supplementary material.}

\section{Numerical experiments} \label{sec:numerical_experiments}

In this section, we present some numerical analyses of our
\textsc{Screenkhorn} algorithm and show how it behaves when
integrated into some complex machine learning pipelines.

\subsection{Setup}

We have implemented our \textsc{Screenkhorn} algorithm in Python and used the L-BFGS-B of
Scipy. Regarding the machine-learning based comparison, we have based our code
on the ones of Python Optimal Transport toolbox (POT)~\citep{flamary2017pot} and just replaced the \texttt{sinkhorn} function call with a \texttt{screenkhorn} one. We have considered the POT's default \textsc{Sinkhorn} stopping criterion parameters and for \textsc{Screenkhorn}, the L-BFGS-B algorithm is stopped when the 
largest component of the projected gradient is smaller than $10^{-6}$, when the number of iterations {or the} number of objective function evaluations reach $10^{5}$. For all applications, we have set $\eta=1$ unless otherwise specified.

\subsection{Analysing on toy problem}
\label{subsec:analysing_toy_problem}

We compare \textsc{Screenkhorn} to \textsc{Sinkhorn} as implemented in POT toolbox\footnote{\url{https://pot.readthedocs.io/en/stable/index.html}} on  a synthetic example. The dataset we use consists of source samples generated from a bi-dimensional gaussian mixture and target samples following the same distribution but with different gaussian means. We consider an unsupervised domain adaptation using optimal transport with entropic regularization.  Several settings are explored: different values of $\eta$, the regularization parameter, the allowed budget $\frac{n_b}{n} = \frac{m_b}{m}$ ranging from $0.01$ to $0.99$, different values of $n$ and $m$.
 We empirically measure  marginal violations as the norms $\norm{{\mu} -{\mu}^{\text{sc}}}_1$ and $\norm{{\nu} -{\nu}^{\text{sc}}}_1$, running time expressed as $\frac{T_{\textsc{Sinkhorn}}}{T_{\text{\textsc{Screenkhorn}}}}$ and the relative divergence difference $| \inr{C, P^\star} - \inr{C, P^{\text{sc}}}|/\inr{C, P^\star}$ between \textsc{Screenkhorn} and \textsc{Sinkhorn}, where $P^\star = \Delta(e^{u^\star}) K \Delta(e^{v^\star})$ and $P^{\text{sc}} = \Delta(e^{u^{\text{sc}}}) K \Delta(e^{v^{\text{sc}}}).$
Figure \ref{fig:margin_expe} summarizes the observed behaviors of both algorithms under these settings. We choose to only report results for $n=m=1000$ as we get similar findings for other values of $n$ and $m$. 

\begin{figure*}[t]
	\begin{center}
				~\hfill\includegraphics[width=0.24\textwidth]{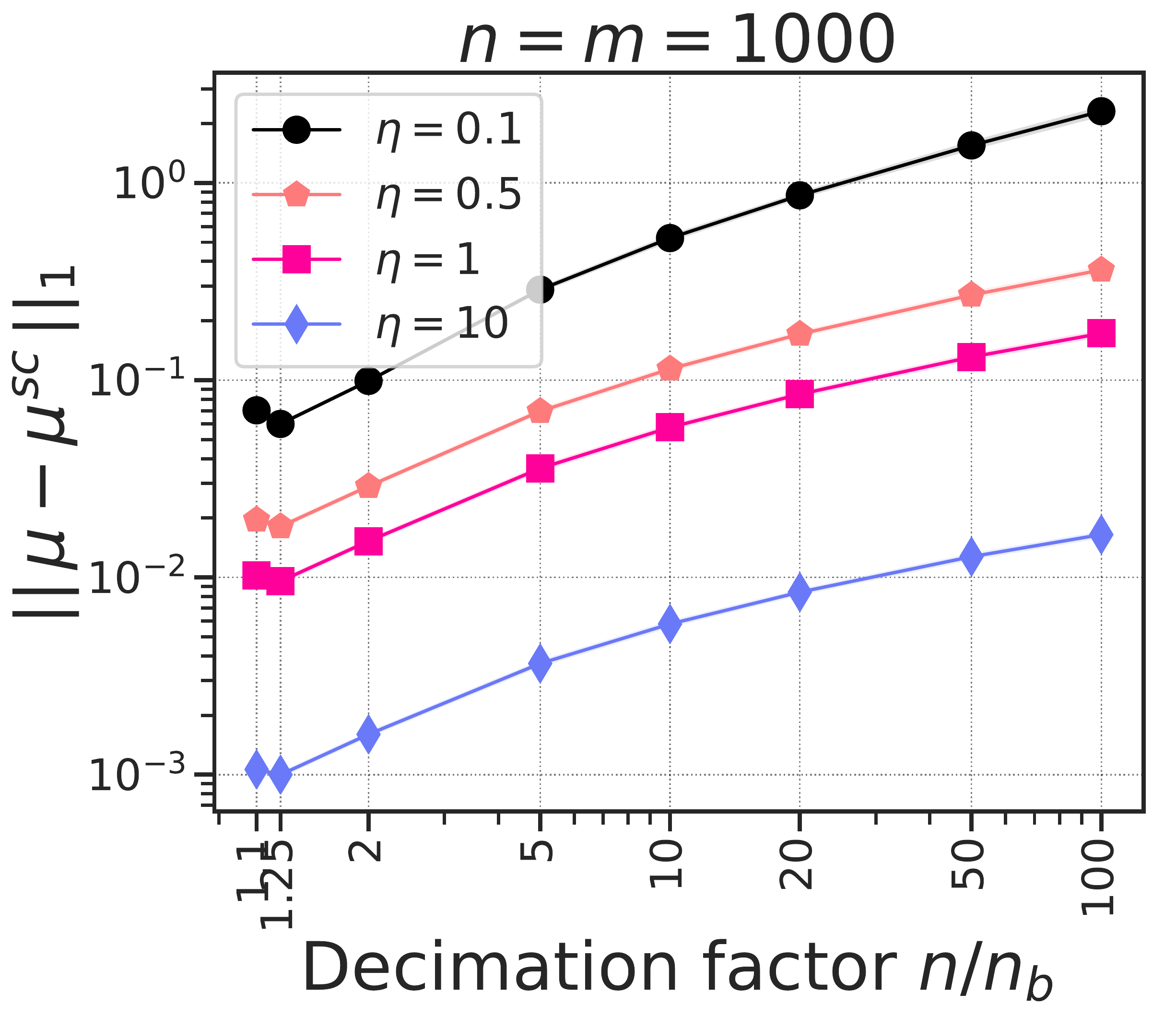}
		 \hfill
		\includegraphics[width=0.24\textwidth]{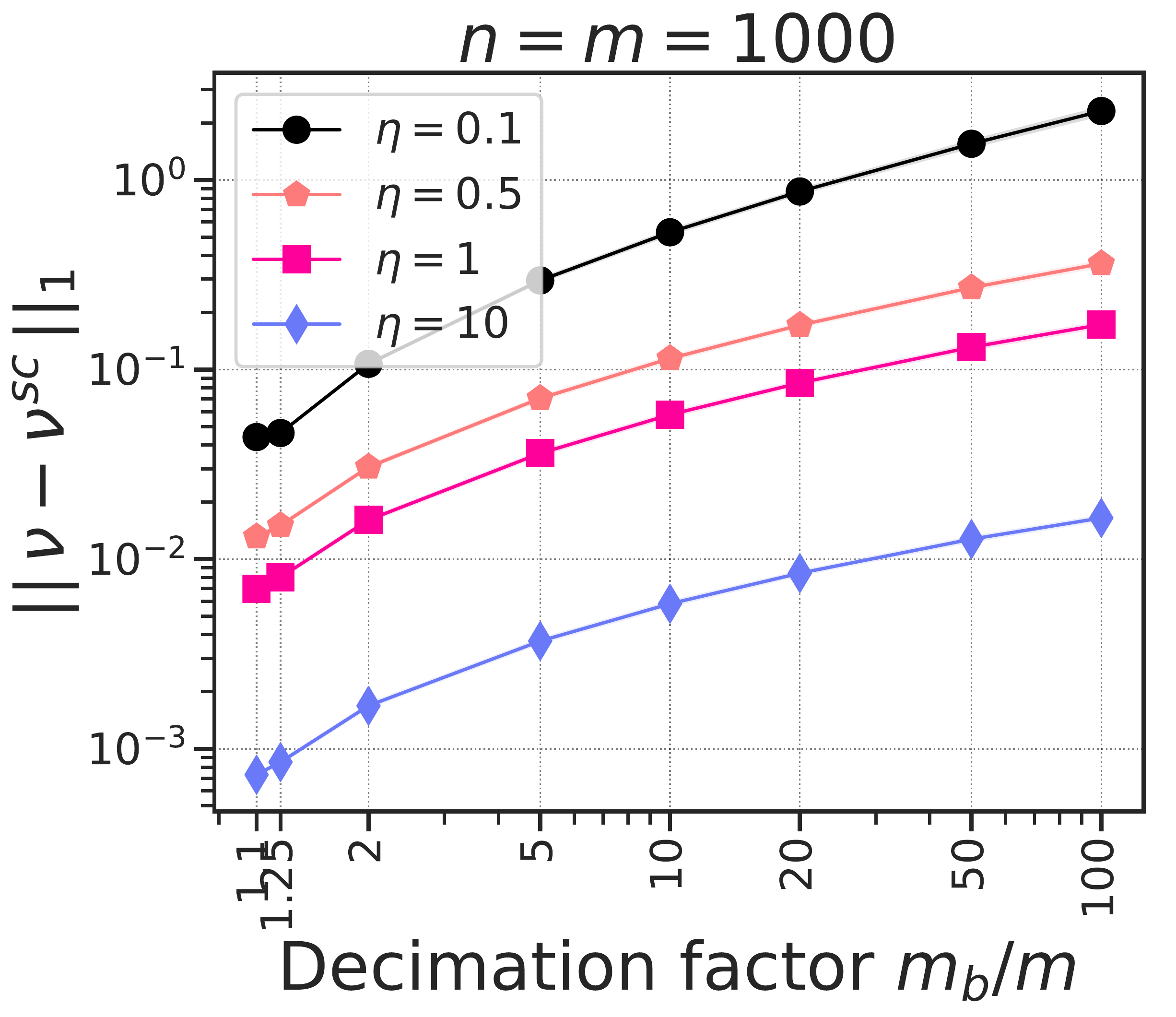} 
		\hfill
		\includegraphics[width=0.24\textwidth]{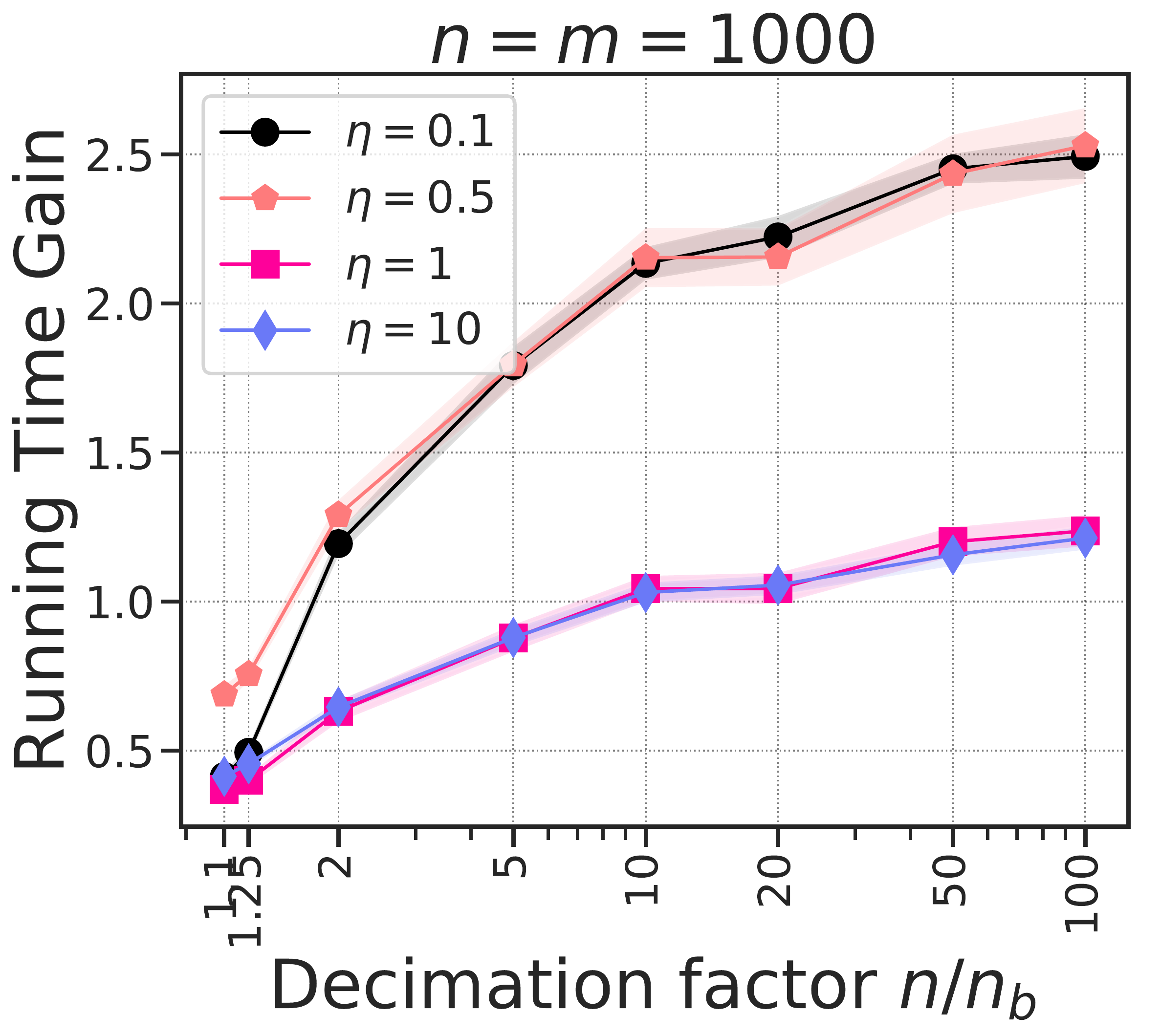}
		\hfill
		\includegraphics[width=0.24\textwidth]{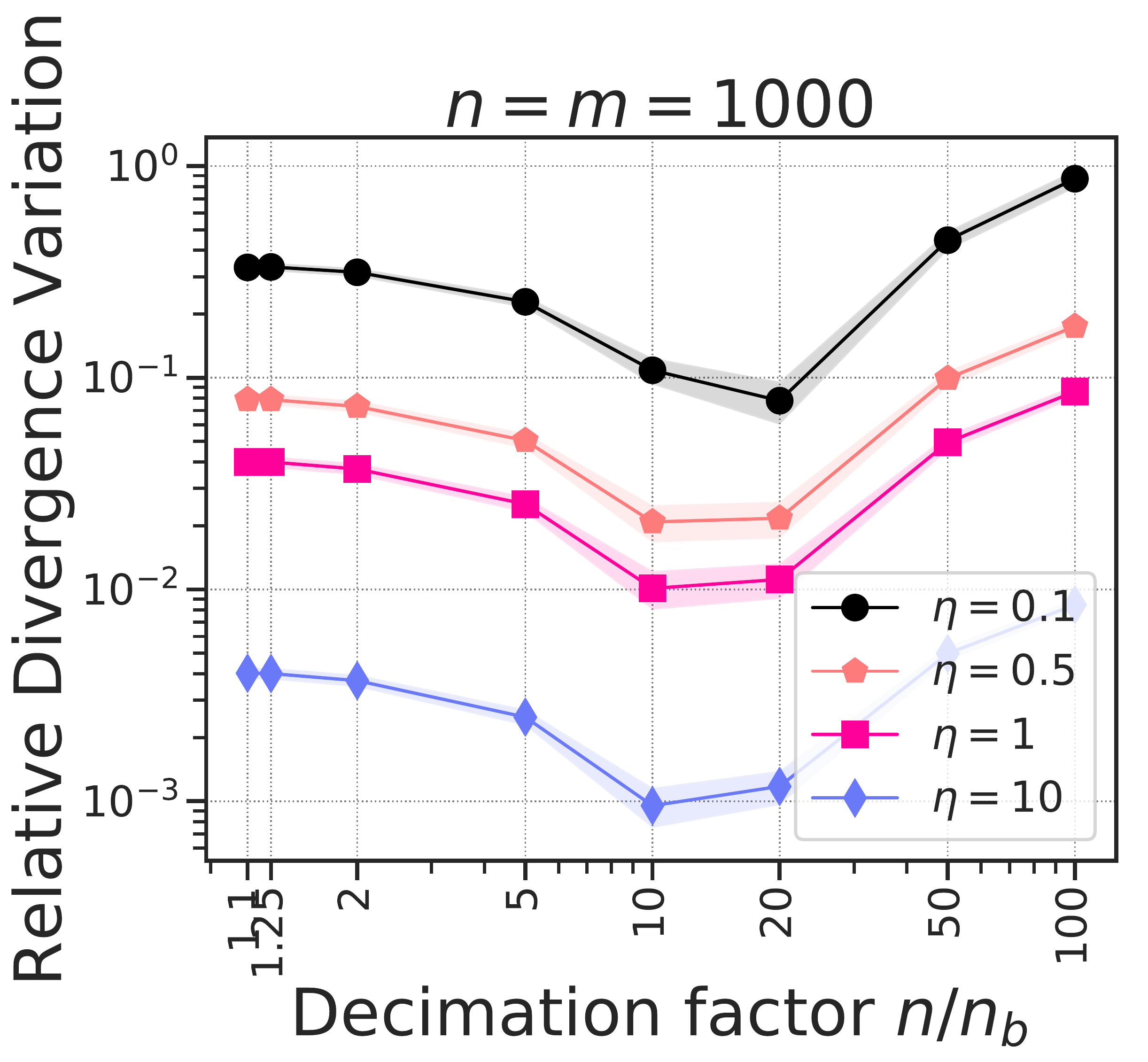}\hfill~
	\end{center}
	\caption{Empirical evaluation of \textsc{Screenkhorn} vs \textsc{Sinkhorn} for normalized cost matrix \emph{i.e.} $\norm{C}_\infty=1$. (most-lefts): marginal violations in relation with the budget of points on $n$ and $m$ .  (center-right) ratio of computation times    $\frac{T_{\textsc{Sinkhorn}}}{T_{\text{\textsc{Screenkhorn}}}}$ and, (right) relative divergence variation. The results are averaged over $30$ trials.} 
		\label{fig:margin_expe}
\end{figure*}
\textsc{Screenkhorn} provides good approximation of the marginals $\mu$ and $\nu$ for ``high'' values of the regularization parameter $\eta$ ($\eta > 1$). The approximation quality diminishes for small $\eta$. As expected $\norm{{\mu} -{\mu}^{\text{sc}}}_1$ and $\norm{{\nu} -{\nu}^{\text{sc}}}_1$ converge towards zero when increasing the budget of points. Remarkably marginal violations are almost negligible whatever the budget for high $\eta$.  According to computation gain, \textsc{Screenkhorn} is almost  2 times faster than \textsc{Sinkhorn} at high decimation factor $n/n_b$ (low budget) while the reverse holds when $n/n_b$ gets close to 1.  Computational benefit of \textsc{Screenkhorn} also depends on $\eta$ with appropriate values $\eta \leq 1$. Finally except for $\eta=0.1$ \textsc{Screenkhorn} achieves a  divergence $\inr{C, P}$ close to the one of Sinkhorn showing that our static screening test provides a
reasonable approximation of the Sinkhorn divergence. As such, we believe that  \textsc{Screenkhorn} will be practically useful in cases where modest accuracy on the divergence is sufficient. This may be the case of a loss function for a gradient descent method (see next section).

\subsection{Integrating \textsc{Screenkhorn} into machine learning pipelines}

Here, we analyse the impact of using \textsc{Screenkhorn}
instead of \textsc{Sinkhorn} in a complex machine learning pipeline. Our two applications
are a dimensionality reduction technique, denoted as Wasserstein Discriminant Analysis (WDA), based on Wasserstein distance approximated
through Sinkhorn divergence \citep{flamary2018WDA} and a domain-adaptation using optimal transport mapping \citep{courty2017optimal}, named OTDA. 

WDA aims at finding a linear projection which minimize the ratio of distance between intra-class samples and distance inter-class samples, where the distance is understood
in a Sinkhorn divergence sense. We have used a toy problem involving Gaussian classes with $2$ discriminative features and $8$ noisy features and the MNIST dataset. For the
former problem, we aim at find the best two-dimensional linear subspace in a WDA sense whereas for MNIST, we look for a subspace of dimension $20$ starting from the original
$728$ dimensions.  Quality of the retrieved subspace are evaluated using classification task based on a $1$-nearest neighbour approach.

Figure \ref{fig:wda} presents the average gain (over $30$ trials) in computational time we get as the number of examples evolve and for different decimation factors of the \textsc{Screenkhorn} problem.
Analysis of the quality of the subspace have been deported to the supplementary material (see Figure~~\ref{fig:wda_gain}), but we can remark a small loss of performance of \textsc{Screenkhorn} for the toy problem, while
for MNIST, accuracies are equivalent regardless of the decimation factor.  We can note
that the minimal gains are respectively $2$ and $4.5$ for the toy and MNIST problem
whereas the maximal gain for $4000$ samples is slightly larger than an order of magnitude. 

\begin{figure*}[t]
	\centering
			~\hfill
	\includegraphics[width=0.37\textwidth]{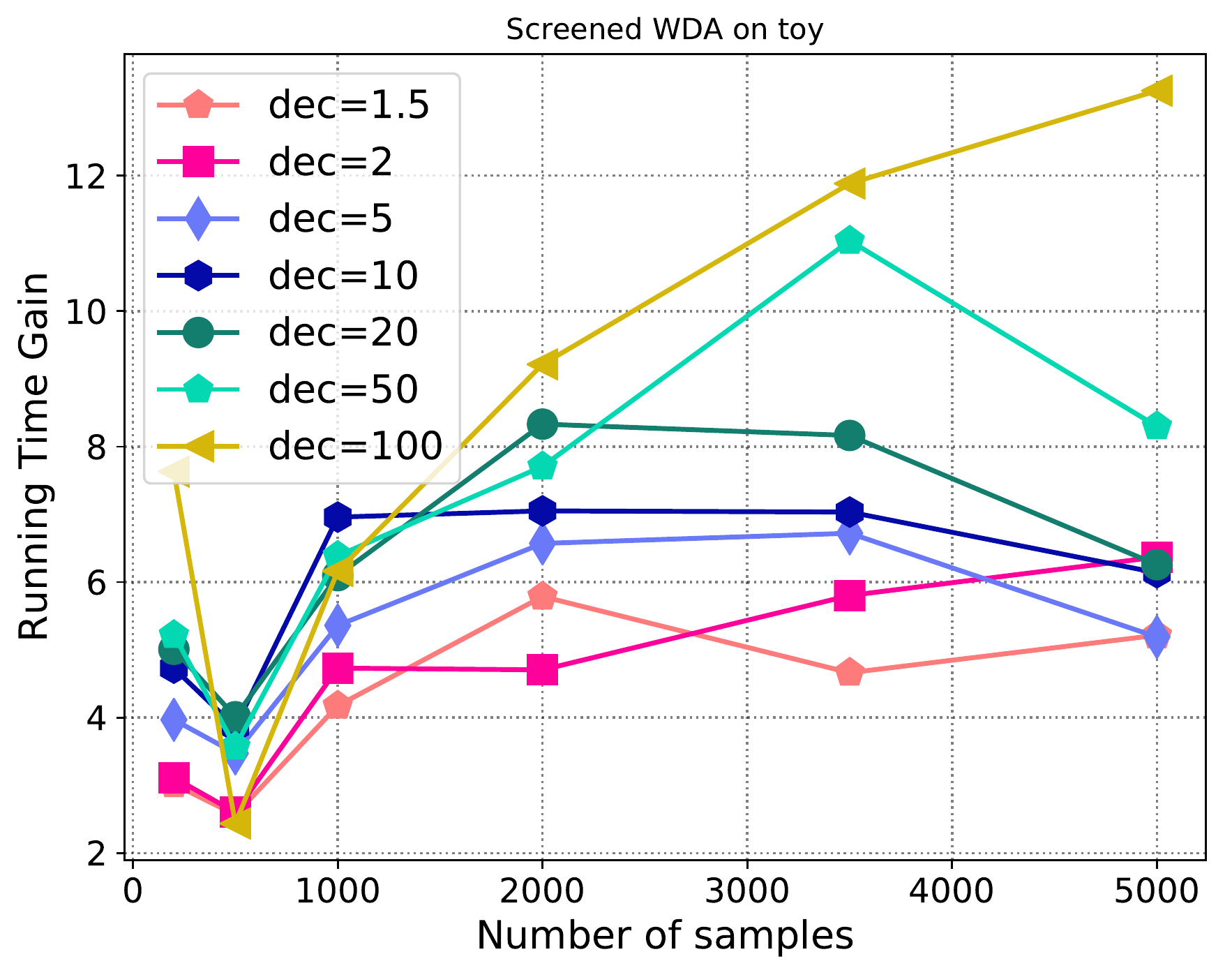}~\hfill~
	\includegraphics[width=0.37\textwidth]{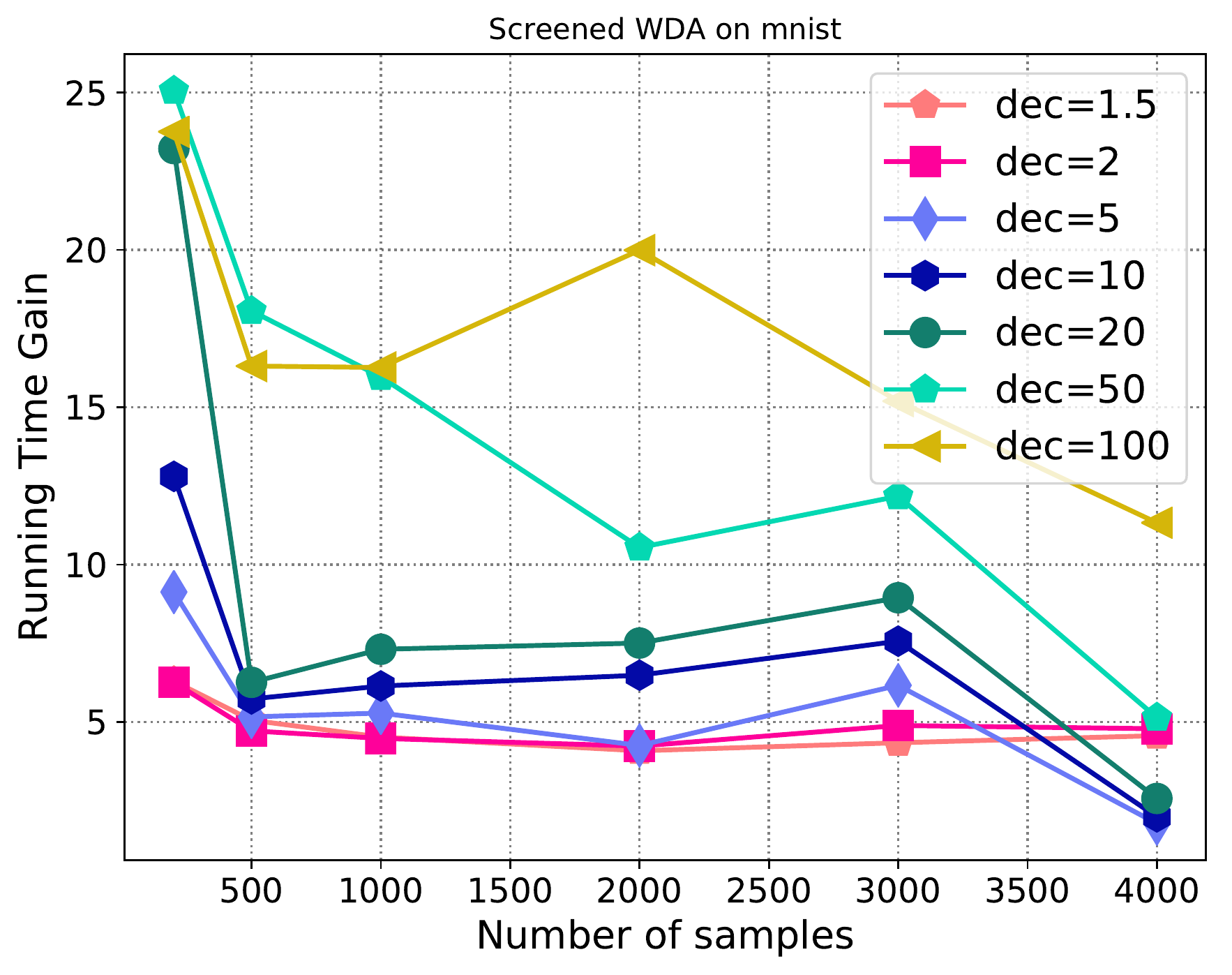}
    \hfill~
	\caption{Wasserstein Discriminant Analysis : running time gain for (left) a toy dataset and (right) MNIST as a function of the number of examples and the data decimation factor in \textsc{Screenkhorn}.}
	\label{fig:wda}
\end{figure*}
\begin{figure*}[t]
	\centering

			~\hfill\includegraphics[width=0.37\textwidth]{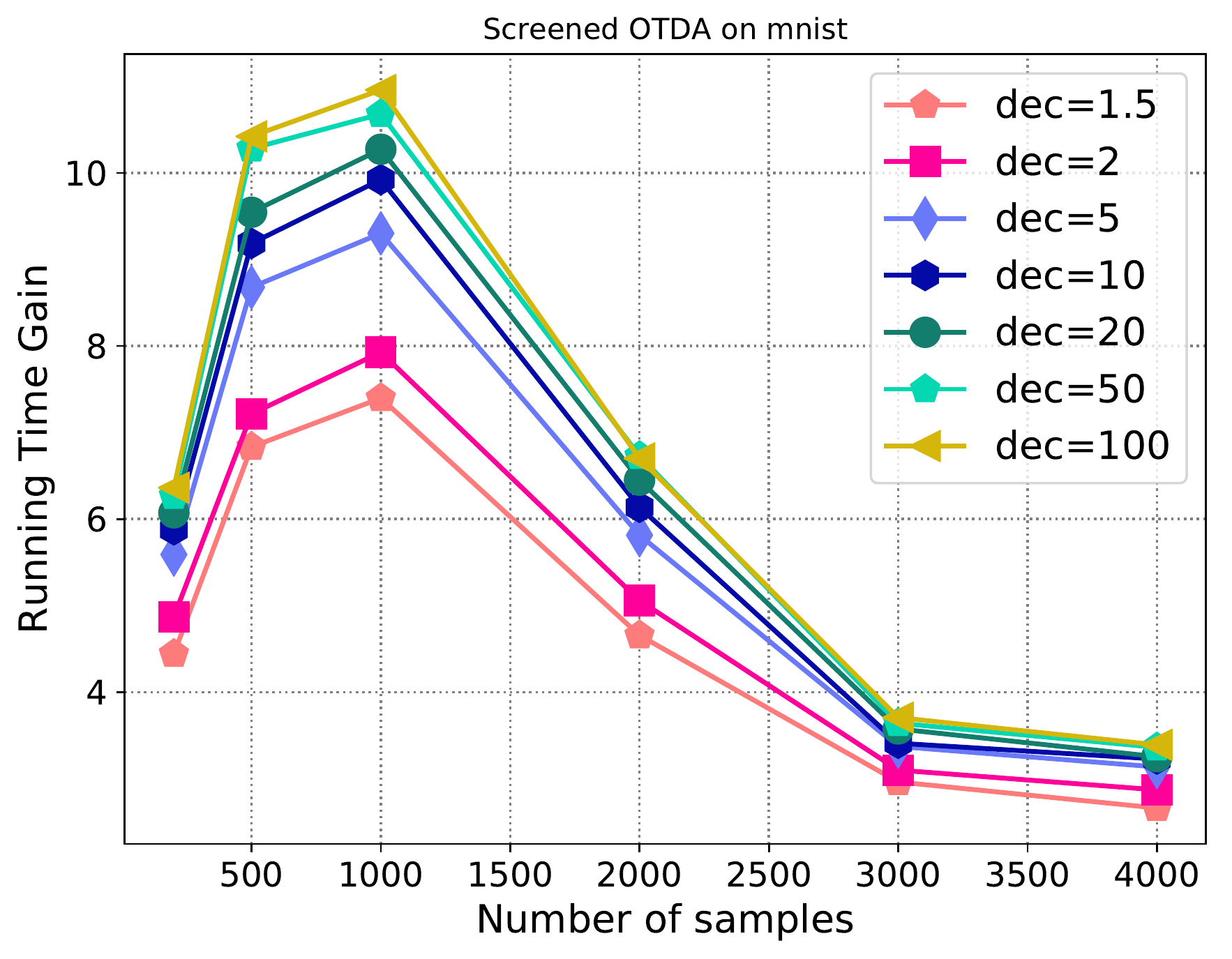}~\hfill~
	\includegraphics[width=0.37\textwidth]{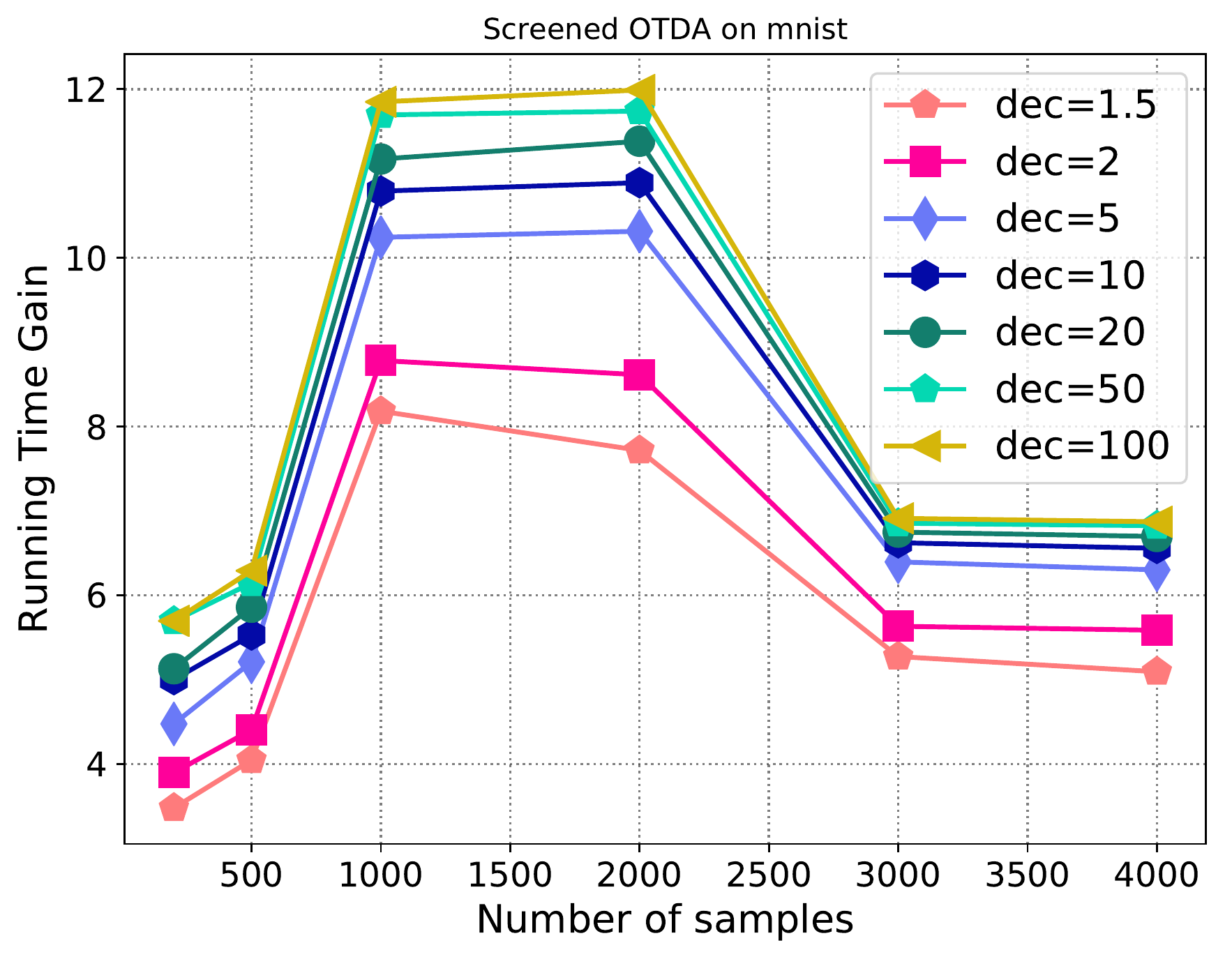}\hfill~
	\caption{OT Domain adaptation : running time gain  for MNIST as a function of the number of examples and the data decimation factor in \textsc{Screenkhorn}. Group-lasso hyperparameter values (left) $1$. (right) $10$.}
	\label{fig:otda}
\end{figure*}

For the OT based domain adaptation problem, we have considered the
OTDA with $\ell_{\frac 12,1}$ group-lasso regularizer that helps in exploiting available labels in the source domain. The problem is solved using a majorization-minimization approach 
for handling the non-convexity of the problem. Hence, at each iteration, a \textsc{Sinkhorn}/\textsc{Screenkhorn} has to be computed and the number of iteration is
sensitive to the regularizer strength.  As a domain-adaptation problem, we have
used a MNIST to USPS problem in which features have been 
{computed from the first layers}
 of a domain adversarial neural networks \citep{ganin2016domain} before full convergence of the networks (so as to leave room for OT adaptation). 
Figure \ref{fig:otda} reports the gain in running time for $2$ different values
of the group-lasso regularizer hyperparameter, while the curves of performances are
reported in the supplementary material. We can note that for all the  \textsc{Screenkhorn} with different decimation factors, the gain in computation goes from a factor of $4$ to $12$, {without any loss of the accuracy performance.}


\section{Conclusion}

The paper introduces a novel efficient approximation of the Sinkhorn divergence
based on a screening strategy. Screening some of the Sinkhorn dual variables
has been made possible by defining a novel constrained dual problem and by 
carefully analyzing its optimality conditions. From the latter, we derived some
sufficient conditions depending on the ground cost matrix, that some dual variables are smaller than a given threshold. Hence, we need just to solve a restricted
dual Sinkhorn problem using an off-the-shelf L-BFGS-B algorithm. We also provide
some theoretical guarantees of the quality of the approximation with respect to
the number of variables that have been screened. Numerical experiments show 
the behaviour of our \textsc{Screenkhorn} algorithm and computational time gain it can
achieve when integrated in some complex machine learning pipelines.

\subsubsection*{Acknowledgments}

This work was supported by grants from the Normandie Projet GRR-DAISI, European funding FEDER DAISI and OATMIL ANR-17-CE23-0012 Project of the French National Research Agency (ANR).

\bibliographystyle{chicago}

\newpage
\section{Supplementary material}

\subsection{Proof of Lemma~\ref{lemma_actives_sets}}

Since the objective function $\Psi_{\kappa}$ is convex with respect to $(u,v)$, the set of optima of problem~\eqref{screen-sinkhorn} is non empty.
Introducing two dual variables $\lambda \in \R^n_{+}$ and $\beta \in \R^m_{+}$ for each constraint, the Lagrangian of problem~\eqref{screen-sinkhorn} reads as 
\begin{equation*}
  \mathscr{L}(u,v, \lambda, \beta) = \frac \varepsilon\kappa\inr{\lambda, \mathbf{1}_n} + \varepsilon\kappa\inr{\beta, \mathbf{1}_m} + \mathbf{1}_n^\top B(u,v) \mathbf{1}_m - \inr{\kappa u, \mu} - \inr{\frac v\kappa, \nu} -\inr{\lambda,e^{u}} - \inr{\beta,e^{v}}
\end{equation*}
First order conditions then yield that the Lagrangian multiplicators solutions $\lambda^{*}$ and $\beta^{*}$ satisfy 
\begin{align*}
  &\nabla_u\mathscr{L}(u^{*},v^{*}, \lambda^{*}, \beta^{*})=  e^{u^{*}} \odot(Ke^{v^{*}} - \lambda^{*}) - \kappa\mu = \mathbf 0_n,\\
  & \text{ and } \nabla_v\mathscr{L}(u^{*},v^{*}, \lambda^{*}, \beta^{*})=  e^{v^{*}} \odot(K^\top e^{u^{*}} - \beta) - \frac \nu\kappa = \mathbf 0_m
\end{align*}
which leads to 
\begin{align*}
  &\lambda^{*} = K e^{v^{*}} - \kappa\mu \oslash e^{u^{*}} \text{ and }
  \beta^{*} = K^\top e^{u^{*}} - \nu \oslash \kappa e^{v^{*}}
\end{align*}

For all $i=1, \ldots, n$ we have that $e^{u^{*}_i} \geq \frac\varepsilon\kappa$. Further, the condition on the dual variable $\lambda^{*}_i > 0$  ensures that $e^{u^{*}_i} = \frac\varepsilon\kappa$ and hence $i \in I^\complement_{\varepsilon,\kappa}$. We have that $\lambda^{*}_i > 0$ is equivalent to $e^{u^{*}_i}r_i(K) e^{v^{*}_j} >  \kappa{\mu_i}$ which  is satisfied when $\varepsilon^2r_i(K) >  \kappa{\mu_i}.$  
In a symmetric way we can prove the same statement for $e^{v^{*}_j}$.

\subsection{Proof of Proposition~\ref{prop:bounds_of_usc_and_vsc}}

We prove only the first statement~\eqref{bound_on_u} and similarly we can prove the second one~\eqref{bound_on_v}.
For all $i\in I_{\varepsilon,\kappa}$, we have $e^{u^{\text{sc}}_i} > \frac \varepsilon\kappa$ or $e^{u^{\text{sc}}_i} = \frac \varepsilon\kappa$. In one hand, if $e^{u^{\text{sc}}_i} > \frac \varepsilon\kappa$ then according to the optimality conditions $\lambda^{\text{sc}}_i = 0,$ which implies $e^{u^{\text{sc}}_i} \sum_{j=1}^m K_{ij} e^{v^{\text{sc}}_j} = \kappa\mu_i$.
In another hand, we have 
\begin{align*}
e^{u^{\text{sc}}_i} \min_{i,j}K_{ij} \sum_{j=1}^m e^{v^{\text{sc}}_j} \leq e^{u^{\text{sc}}_i} \sum_{j=1}^m K_{ij} e^{v^{\text{sc}}_j} = \kappa\mu_i.
\end{align*}
We further observe that $\sum_{j=1}^m e^{v^{\text{sc}}_j} = \sum_{j \in J_{\varepsilon,\kappa}} e^{v^{\text{sc}}_j} + \sum_{j \in J^\complement_{\varepsilon,\kappa}} e^{v^{\text{sc}}_j} \geq \varepsilon\kappa |J_{\varepsilon,\kappa}| + \varepsilon\kappa |J^\complement_{\varepsilon,\kappa}|=\varepsilon\kappa m.$ Then
\begin{equation*}
\max_{i\in I_{\varepsilon,\kappa}} e^{u^{\text{sc}}_i} \leq \frac{\max_{i\in I_{\varepsilon,\kappa}}\mu_i}{m\varepsilon K_{\min}}.
\end{equation*}
Analogously, one can obtain for all $j\in J_{\varepsilon,\kappa}$
\begin{equation}
\label{upper_bound_v_potential}
\max_{j\in J_{\varepsilon,\kappa}}e^{v^{\text{sc}}_j} \leq  \frac{\max_{j \in J_{\varepsilon,\kappa}} \nu_j}{n\varepsilon K_{\min}}.
\end{equation}

Now, since $K_{ij} \leq 1$, we have 
\begin{align*}
e^{u^{\text{sc}}_i} \sum_{j=1}^m e^{v^{\text{sc}}_j} \geq e^{u^{\text{sc}}_i} \sum_{j=1}^m K_{ij}e^{v^{\text{sc}}_j} = \kappa\mu_i.
\end{align*}
Using~\eqref{upper_bound_v_potential}, we get 
\begin{align*}
\sum_{j=1}^m e^{v^{\text{sc}}_j} &= \sum_{j \in J_{\varepsilon,\kappa}} e^{v^{\text{sc}}_j} + \sum_{j \in J^\complement_{\varepsilon,\kappa}} e^{v^{\text{sc}}_j}
\leq \varepsilon\kappa |J^\complement_{\varepsilon,\kappa}| + \frac{\max_{j\in J_{\varepsilon,\kappa}} \nu_j}{n\varepsilon K_{\min}} |J_{\varepsilon,\kappa}|.
\end{align*}
Therefore,
\begin{align*}
\min_{i \in I_{\varepsilon,\kappa}} e^{u^{\text{sc}}_i}  \geq \frac \varepsilon\kappa \vee \frac{\kappa\min_{I_{\varepsilon,\kappa}}\mu_i}{\varepsilon\kappa (m-m_b) + \frac{\max_{j\in J_{\varepsilon,\kappa}} \nu_j}{n\varepsilon K_{\min}} m_b}.
\end{align*}

\subsection{Proof of Proposition~\ref{proposition_error_in_marginals}}

We define the distance function $\varrho: \R_+ \times \R_+ \mapsto [0, \infty]$ by $\varrho(a,b) = b - a + a \log(\frac ab).$
While $\varrho$ is not a metric, it is easy to see that $\varrho$ is not nonnegative and satisfies $\varrho(a,b) =0$ iff $a=b$.
The violations are computed through the following function: 
\begin{equation*}
	d_{\varrho}(\gamma,\beta) = \sum_{i=1}^n \varrho(\gamma_i,\beta_i), \text{ for } \gamma, \beta \in \R^n_+.
\end{equation*}
Note that if $\gamma,\beta$ are two vectors of positive entries, $d_{\varrho}(\gamma,\beta)$ will return some measurement on how far they are from each other. The next Lemma is from~\cite{khalilabid2018} (see Lemma 7 herein).
\begin{lemma}
\label{lem:pinsker}
For any $\gamma, \beta \in \R^n_+$, the following generalized Pinsker inequality holds 
\begin{align*}
\norm{\gamma - \beta}_1 \leq \sqrt{7 (\norm{\gamma}_1\wedge \norm{\beta}_1)d_{\varrho}(\gamma,\beta)}.
\end{align*}
\end{lemma}

The optimality conditions for $({u}^{\text{sc}}, {v}^{\text{sc}})$ entails 
\begin{align}
\label{i-th-marginal-mu} 
{\mu}^{\text{sc}}_i  &= 
\begin{cases}
e^{u^{\text{sc}}_i} \sum_{j=1}^m K_{ij} e^{v^{\text{sc}}_j}, \text{ if  }i \in I_{\varepsilon,\kappa},\\
\frac \varepsilon\kappa\sum_{j=1}^m K_{ij} e^{v^{\text{sc}}_j}, \text{ if  }i \in I^\complement_{\varepsilon,\kappa}
\end{cases}
=\begin{cases}
\kappa \mu_i, \text{ if  }i \in I_{\varepsilon,\kappa},\\
\frac \varepsilon\kappa\sum_{j=1}^m K_{ij} e^{v^{\text{sc}}_j}, \text{ if  }i \in I^\complement_{\varepsilon,\kappa},
\end{cases}
\end{align}
and 
\begin{align}
\label{i-th-marginal-nu}
{\nu}^{\text{sc}}_j  &= 
\begin{cases}
e^{v^{\text{sc}}_j} \sum_{i=1}^n K_{ij} e^{u^{\text{sc}}_i}, \text{ if  }j \in J_{\varepsilon,\kappa},\\
\varepsilon\kappa\sum_{i=1}^n K_{ij} e^{u^{\text{sc}}_i}, \text{ if  }j \in J^\complement_{\varepsilon,\kappa}
\end{cases}
=\begin{cases}
\frac{\nu_j}{\kappa}, \text{ if  }j \in J_{\varepsilon,\kappa},\\
\varepsilon\kappa\sum_{i=1}^n K_{ij} e^{u^{\text{sc}}_i}, \text{ if  }j \in J^\complement_{\varepsilon,\kappa}.
\end{cases}
\end{align}

By~\eqref{i-th-marginal-mu}, we have
\begin{align*}
d_\varrho({\mu} ,{\mu}^{\text{sc}}) &= \sum_{i=1}^n  {\mu}^{\text{sc}}_i - {\mu}_i + {\mu}_i  \log\Big(\frac{{\mu}_i}{{\mu}^{\text{sc}}_i }\Big)\\
&= \sum_{i\in I_{\varepsilon,\kappa}} (\kappa-1)\mu_i - \mu_i\log(\kappa) + \sum_{i\in I^\complement_{\varepsilon,\kappa}}\frac \varepsilon\kappa\sum_{j=1}^m K_{ij} e^{v^{\text{sc}}_j} - \mu_i + \mu_i \log\Big(\frac{\mu_i}{\frac \varepsilon\kappa\sum_{j=1}^m K_{ij} e^{v^{\text{sc}}_j}}\Big)\\
&= \sum_{i\in I_{\varepsilon,\kappa}} (\kappa-\log(\kappa)-1)\mu_i  + \sum_{i\in I^\complement_{\varepsilon,\kappa}}\frac \varepsilon\kappa\sum_{j=1}^m K_{ij} e^{v^{\text{sc}}_j} - \mu_i + \mu_i \log\Big(\frac{\mu_i}{\frac \varepsilon\kappa\sum_{j=1}^m K_{ij} e^{v^{\text{sc}}_j}}\Big).
\end{align*}
Now by~\eqref{bound_on_v}, we have in one hand 
\begin{align*}
\sum_{i\in I^\complement_{\varepsilon,\kappa}}\frac \varepsilon\kappa\sum_{j=1}^m K_{ij} e^{v^{\text{sc}}_j}&= \sum_{i\in I^\complement_{\varepsilon,\kappa}}\frac \varepsilon\kappa \Big(\sum_{j\in J_{\varepsilon,\kappa}}K_{ij} e^{v^{\text{sc}}_j} + \varepsilon \kappa\sum_{j\in J^\complement_{\varepsilon,\kappa}}K_{ij}\Big)\\
&\leq \sum_{i\in I^\complement_{\varepsilon,\kappa}}\frac \varepsilon\kappa \Big(m_b \max_{i,j}K_{ij}\frac{\max_{j \in J_{\varepsilon,\kappa}} \nu_j}{n\varepsilon K_{\min}} + (m - m_b)\varepsilon\kappa\max_{i,j}K_{ij}\Big) \\
&\leq (n-n_b)\Big(\frac{m_b\max_{j} \nu_j}{n\kappa K_{\min}} + (m- m_b) \varepsilon^2\Big).
\end{align*}
On the other hand, we get
\begin{align*}
\frac \varepsilon\kappa\sum_{j=1}^m K_{ij} e^{v^{\text{sc}}_j}&=\frac \varepsilon\kappa \Big(\sum_{j\in J_{\varepsilon,\kappa}}K_{ij} e^{v^{\text{sc}}_j} + \varepsilon \kappa\sum_{j\in J^\complement_{\varepsilon,\kappa}}K_{ij}\Big)\\
&\geq m_bK_{\min} \frac{m\varepsilon^2K_{\min}\min_{j \in J_{\varepsilon,\kappa}}\nu_j}{\kappa((n-n_b)m\varepsilon^2K_{\min} + m\varepsilon^2K_{\min} + n_b\kappa\max_{i\in I_{\varepsilon,\kappa}}\mu_i)}\\
&\qquad +\varepsilon^2 (m- m_b) K_{\min}\\
&\geq \frac{mm_b\varepsilon^2(K_{\min})^2\min_{j \in J_{\varepsilon,\kappa}}\nu_j}{\kappa((n-n_b)m\varepsilon^2K_{\min}+ m\varepsilon^2K_{\min} + n_b\kappa\max_{i\in I_{\varepsilon,\kappa}}\mu_i)}\\
&\qquad +\varepsilon^2 (m- m_b) K_{\min}\\
&\geq \frac{mm_b\varepsilon^2K_{\min}^2\min_{j \in J_{\varepsilon,\kappa}}\nu_j}{\kappa((n-n_b)m\varepsilon^2K_{\min}+ m\varepsilon^2K_{\min} + n_b\kappa\max_{i\in I_{\varepsilon,\kappa}}\mu_i)}.
\end{align*}
Then 
\begin{align*}
\frac{1}{\frac \varepsilon\kappa\sum_{j=1}^m K_{ij} e^{v^{\text{sc}}_j}} &\leq \frac{\kappa((n-n_b)m\varepsilon^2K_{\min}+ m\varepsilon^2K_{\min} + n_b\kappa\max_{i\in I_{\varepsilon,\kappa}}\mu_i)}{mm_b\varepsilon^2 K_{\min}^2\min_{j \in J_{\varepsilon,\kappa}}\nu_j}\\
&\leq \frac{\kappa(n-n_b+ 1)}{m_bK_{\min}\min_{j \in J_{\varepsilon,\kappa}}\nu_j} + \frac{n_b\kappa^2\max_{i\in I_{\varepsilon,\kappa}}\mu_i}{mm_b\varepsilon^2K_{\min}^2\min_{j \in J_{\varepsilon,\kappa}}\nu_j}.
\end{align*}
It entails 
\begin{align*}
&\sum_{i\in I^\complement_{\varepsilon,\kappa}}\frac \varepsilon\kappa\sum_{j=1}^m K_{ij} e^{v^{\text{sc}}_j} - \mu_i + \mu_i \log\Big(\frac{\mu_i}{\frac \varepsilon\kappa\sum_{j=1}^m K_{ij} e^{v^{\text{sc}}_j}}\Big)\\
&\qquad \leq (n-n_b)\bigg(\frac{m_b}{n\kappa K_{\min}} + (m- m_b) \varepsilon^2 - \min_{i}\mu_i\\
&\qquad \qquad + \max_{i}\mu_i\log\Big(\frac{\kappa(n-n_b+ 1)\max_{i}\mu_i}{m_bK_{\min}\min_{j \in J_{\varepsilon,\kappa}}\nu_j} + \frac{n_b\kappa^2(\max_{i}\mu_i)^2}{mm_b\varepsilon^2 K_{\min}^2\min_{j \in J_{\varepsilon,\kappa}}\nu_j}\Big)
\bigg).
\end{align*}
Therefore
\begin{align*}
d_\varrho({\mu},{\mu}^{\text{sc}}) &\leq n_b c_{\kappa}\max_{i} \mu_i + (n-n_b)\bigg(\frac{m_b\max_{j}\nu_j}{n\kappa K_{\min}} + (m- m_b) \varepsilon^2 - \min_{i}\mu_i\\
&\qquad + \max_{i} \mu_i\log\Big(\frac{\kappa(n-n_b+ 1)\max_{i} \mu_i}{m_bK_{\min}\min_{j \in J_{\varepsilon,\kappa}}\nu_j} + \frac{n_b\kappa^2(\max_{i} \mu_i)^2}{mm_b\varepsilon^2 K_{\min}^2\min_{j \in J_{\varepsilon,\kappa}}\nu_j}\Big).
\end{align*}
Finally, by Lemma~\ref{lem:pinsker} we obtain
\begin{align*}
\norm{{\mu} -{\mu}^{\text{sc}}}^2_1 \leq & n_bc_{\kappa}\max_{i} \mu_i + 7(n-n_b)\bigg(\frac{m_b\max_{j}\nu_j}{n\kappa K_{\min}} + (m- m_b) \varepsilon^2 - \min_{i}\mu_i\\
&+ \max_{i} \mu_i\log\Big(\frac{\kappa(n-n_b+ 1)\max_{i} \mu_i}{m_bK_{\min}\min_{j \in J_{\varepsilon,\kappa}}\nu_j} + \frac{n_b\kappa^2(\max_{i} \mu_i)^2}{mm_b\varepsilon^2K_{\min}^2\min_{j \in J_{\varepsilon,\kappa}}\nu_j}\Big).
\end{align*}
Following the same lines as above, we also have
\begin{align*}
\norm{{\nu} -{\nu}^{\text{sc}}}^2_1 \leq & m_bc_{\frac 1\kappa}\max_{i} \mu_i + 7(m-m_b)\bigg(\frac{n_b\kappa\max_{i}\mu_i}{mK_{\min}} + (n- n_b) \varepsilon^2 - \min_{j}\nu_j\\
&+ \max_{j} \nu_j\log\Big(\frac{(m-m_b+ 1)\max_{j} \nu_j}{n_b\kappa K_{\min}\min_{i \in I_{\varepsilon,\kappa}}\mu_i} + \frac{m_b(\max_{j} \nu_j)^2}{nn_b\varepsilon^2\kappa^2K_{\min}^2\min_{i \in I_{\varepsilon,\kappa}}\mu_i}\Big).\end{align*}
To get the closed forms~\eqref{marginal-error-mu} and~\eqref{marginal-error-nu}, we used the following facts:
\begin{remark}
\label{rem:orders_of_epsilonappa}
 We have $\log(1/K_{\min}^r) = r\norm{C}_\infty /\eta,$ for every $r \in \mathbb{N}$. Using~\eqref{epsilon_kappa}, we further derive: 
$\varepsilon = \bigO((mnK_{\min}^2)^{-1/4})$, $\kappa= \bigO(\sqrt{m/(nc_{\mu\nu}K_{\min})}),$ $\kappa^{-1} = \bigO(\sqrt{n/(mK_{\min}c_{\mu\nu}})$, $(\kappa/\varepsilon)^2 = \bigO( m^{3/2} / \sqrt{nK_{\min}}(c_{\mu\nu})^{3/2})$, and $(\varepsilon\kappa)^{-2} = \bigO(n^{3/2}/\sqrt{mK_{\min}}c^{3/2}_{\mu\nu}).$
\end{remark}

\subsection{Proof of Proposition~\ref{prop:objective-error}}

We first define $\widetilde{K}$ a rearrangement of $K$ with respect to the active sets $I_{\varepsilon,\kappa}$ and $J_{\varepsilon,\kappa}$as follows:
\begin{equation*}
\widetilde{K} = 
\begin{bmatrix}
K_{(I_{\varepsilon,\kappa}, J_{\varepsilon,\kappa})} & K_{(I_{\varepsilon,\kappa}, J^\complement_{\varepsilon,\kappa})}\\
K_{(I^\complement_{\varepsilon,\kappa}, J_{\varepsilon,\kappa})} &K_{(I^\complement_{\varepsilon,\kappa}, J^\complement_{\varepsilon,\kappa})}
\end{bmatrix}.
\end{equation*}
Setting $\dt{\mu} = (\mu_{I_{\varepsilon,\kappa}}^\top, \mu_{I^\complement_{\varepsilon,\kappa}}^\top)^\top$, $\dt{\nu} = (\nu_{J_{\varepsilon,\kappa}}^\top, \nu_{J^\complement_{\varepsilon,\kappa}}^\top)^\top$ and for each vectors $u \in \R^n$ and $v\in \R^m$ we set $\dt{u} =(u_{I_{\varepsilon,\kappa}}^\top, u_{I^\complement_{\varepsilon,\kappa}}^\top)^\top  \text{ and } \dt{v} =(v_{J_{\varepsilon,\kappa}}^\top, v_{J^\complement_{\varepsilon,\kappa}}^\top)^\top.$
We then have 
\begin{equation*}
\Psi_{\varepsilon, \kappa} (u,v) = \mathbf 1_n^\top \widetilde{B}(\dt{u}, \dt{v}) \mathbf 1_m - \kappa \dt{\mu}^\top \dt{u} - \kappa^{-1} \dt{\nu}^\top \dt{v},
\end{equation*}
and 
\begin{equation*}
\Psi_{} (u,v) = \mathbf 1_n^\top \widetilde{B}(\dt{u}, \dt{v}) \mathbf 1_m - \dt{\mu}^\top \dt{u} - \dt{\nu}^\top \dt{v},
\end{equation*}
where
\begin{equation*}
  \widetilde{B}(\dt{u}, \dt{v}) = \Delta(e^{\dt{u}})\widetilde{K} \Delta(e^{\dt{v}}).
\end{equation*}
Let us consider the convex function
\begin{equation*}
	(\hat u, \hat v) \mapsto \inr{\mathbf 1_n, \widetilde{B}(\dt{\hat u}^{\text{}} ,\dt{\hat v}^{\text{}})\mathbf 1_m} - \inr{\kappa\dt{\hat u}, \widetilde{B}(\dt{u}^{\text{sc}} ,\dt{v}^{\text{sc}})\mathbf 1_m} - \inr{\kappa^{-1}\dt{\hat v}, \widetilde{B}(\dt{u}^{\text{sc}} ,\dt{v}^{\text{sc}})^\top\mathbf 1_n}.
\end{equation*}
Gradient inequality of any convex function g at point $x_o$ reads as $g(x_o) \geq g(x) + \inr{\nabla g(x), x_o - x}, \text{ for all } x \in \textbf{dom}(g).$
Applying the latter fact to the above function at point ($u^{\star}, v^{\star})$ we obtain
\begin{align*}
\inr{\mathbf 1_n, \widetilde{B}(\dt{u}^{\text{sc}} ,\dt{v}^{\text{sc}})\mathbf 1_m} &- \inr{\kappa\dt{u}^{\text{sc}}, \widetilde{B}(\dt{u}^{\text{sc}} ,\dt{v}^{\text{sc}})\mathbf 1_m} - \inr{\kappa^{-1}\dt{v}^{\text{sc}}, \widetilde{B}(\dt{u}^{\text{sc}} ,\dt{v}^{\text{sc}})^\top\mathbf 1_n}\\
& - \big(\inr{\mathbf 1_n, \widetilde{B}(\dt{u}^{\star} ,\dt{v}^{\star})\mathbf 1_m}  
- \inr{\kappa\dt{u}^{\star{}}, \widetilde{B}(\dt{u}^{\text{sc}} ,\dt{v}^{\text{sc}})\mathbf 1_m} 
- \inr{\kappa^{-1}\dt{v}^\star, \widetilde{B}(\dt{u}^{\text{sc}} ,\dt{v}^{\text{sc}})^\top\mathbf 1_n} \big)\\
&\qquad \qquad  \leq \inr{\dt{u}^{\text{sc}} - \dt{u}^\star, (1-\kappa) \widetilde{B}(\dt{u}^{\text{sc}} ,\dt{v}^{\text{sc}})\mathbf 1_m} 
+ \inr{\dt{v}^{\text{sc}} - \dt{v}^\star, (1-\kappa^{-1}) \widetilde{B}(\dt{u}^{\text{sc}} ,\dt{v}^{\text{sc}})^\top\mathbf 1_n}.
\end{align*}
Moreover,
\begin{align*}
\Psi_{\varepsilon, \kappa} (u^{\text{sc}} ,v^{\text{sc}}) -\Psi(u^\star, v^\star)
&= \inr{\mathbf 1_n, \widetilde{B}(\dt{u}^{\text{sc}} ,\dt{v}^{\text{sc}})\mathbf 1_m} - \inr{\kappa\dt{u}^{\text{sc}}, \widetilde{B}(\dt{u}^{\text{sc}} ,\dt{v}^{\text{sc}})\mathbf 1_m} - \inr{\kappa^{-1}\dt{v}^{\text{sc}}, \widetilde{B}(\dt{u}^{\text{sc}}, \dt{v}^{\text{sc}})\mathbf 1_n^\top}\\
&\qquad - \big( 
\inr{\mathbf 1_n, \widetilde{B}(\dt{u}^{\star} ,\dt{v}^{\star})\mathbf 1_m}  
- \inr{\dt{u}^{\star{}}, \widetilde{B}(\dt{u}^{\text{sc}} ,\dt{v}^{\text{sc}})\mathbf 1_m} 
- \inr{\dt{v}^\star, \widetilde{B}(\dt{u}^{\text{sc}} ,\dt{v}^{\text{sc}})^\top\mathbf 1_n} \big)\\
&\qquad + \inr{\kappa \dt{u}^{\text{sc}} - \dt{u}^{\star}, \widetilde{B}(\dt{u}^{\text{sc}} ,\dt{v}^{\text{sc}})\mathbf 1_m - \dt{\mu}} + \inr{\kappa^{-1}\dt{v}^{\text{sc}} - \dt{v}^{\star}, \widetilde{B}(\dt{u}^{\text{sc}} ,\dt{v}^{\text{sc}})^\top\mathbf 1_n - \dt{\nu}}.
\end{align*}
Hence,
\begin{align*}
\Psi_{\varepsilon, \kappa} (u^{\text{sc}} ,v^{\text{sc}}) -\Psi(u^\star, v^\star)
+& \big( \inr{\mathbf 1_n, \widetilde{B}(\dt{u}^{\star} ,\dt{v}^{\star})\mathbf 1_m} \\
&- \inr{\dt{u}^{\star{}}, \widetilde{B}(\dt{u}^{\text{sc}} ,\dt{v}^{\text{sc}})\mathbf 1_m} 
- \inr{\dt{v}^\star, \widetilde{B}(\dt{u}^{\text{sc}} ,\dt{v}^{\text{sc}})^\top\mathbf 1_n} \big)\\
& \qquad -\inr{\kappa \dt{u}^{\text{sc}} - \dt{u}^{\star}, \widetilde{B}(\dt{u}^{\text{sc}} ,\dt{v}^{\text{sc}})\mathbf 1_m - \dt{\mu}} -\inr{\kappa^{-1}\dt{v}^{\text{sc}} - \dt{v}^{\star}, \widetilde{B}(\dt{u}^{\text{sc}} ,\dt{v}^{\text{sc}})^\top\mathbf 1_n - \dt{\nu}}\\
& \leq \inr{\dt{u}^{\text{sc}} - \dt{u}^\star, (1-\kappa) \widetilde{B}(\dt{u}^{\text{sc}} ,\dt{v}^{\text{sc}})\mathbf 1_m} 
+ \inr{\dt{v}^{\text{sc}} - \dt{v}^\star, (1-\kappa^{-1}) \widetilde{B}(\dt{u}^{\text{sc}} ,\dt{v}^{\text{sc}})^\top\mathbf 1_n}\\
&\qquad + \big(\inr{\mathbf 1_n, \widetilde{B}(\dt{u}^{\star} ,\dt{v}^{\star})\mathbf 1_m}  
- \inr{\kappa\dt{u}^{\star{}}, \widetilde{B}(\dt{u}^{\text{sc}} ,\dt{v}^{\text{sc}})\mathbf 1_m} 
- \inr{\kappa^{-1}\dt{v}^\star, \widetilde{B}(\dt{u}^{\text{sc}} ,\dt{v}^{\text{sc}})^\top\mathbf 1_n} \big).
\end{align*}
Then,
\begin{align*}
\Psi_{\varepsilon, \kappa} (u^{\text{sc}} ,v^{\text{sc}}) -\Psi(u^\star, v^\star)
&\leq \inr{\dt{u}^{\text{sc}} - \dt{u}^\star, (1-\kappa) \widetilde{B}(\dt{u}^{\text{sc}} ,\dt{v}^{\text{sc}})\mathbf 1_m} 
+ \inr{\dt{v}^{\text{sc}} - \dt{v}^\star, (1-\kappa^{-1}) \widetilde{B}(\dt{u}^{\text{sc}} ,\dt{v}^{\text{sc}})^\top\mathbf 1_n}\\
&\qquad + \big(\inr{\mathbf 1_n, \widetilde{B}(\dt{u}^{\star} ,\dt{v}^{\star})\mathbf 1_m}  
- \inr{\kappa\dt{u}^{\star{}}, \widetilde{B}(\dt{u}^{\text{sc}} ,\dt{v}^{\text{sc}})\mathbf 1_m} 
- \inr{\kappa^{-1}\dt{v}^\star, \widetilde{B}(\dt{u}^{\text{sc}} ,\dt{v}^{\text{sc}})^\top\mathbf 1_n} \big)
\\
&\qquad + \inr{\kappa \dt{u}^{\text{sc}} - \dt{u}^{\star}, \widetilde{B}(\dt{u}^{\text{sc}} ,\dt{v}^{\text{sc}})\mathbf 1_m - \dt{\mu}} + \inr{\kappa^{-1}\dt{v}^{\text{sc}} - \dt{v}^{\star}, \widetilde{B}(\dt{u}^{\text{sc}} ,\dt{v}^{\text{sc}})^\top\mathbf 1_n - \dt{\nu}}\\
&\qquad - \big(\inr{\mathbf 1_n, \widetilde{B}(\dt{u}^{\star} ,\dt{v}^{\star})\mathbf 1_m}  
- \inr{\dt{u}^{\star{}}, \widetilde{B}(\dt{u}^{\text{sc}} ,\dt{v}^{\text{sc}})\mathbf 1_m} 
- \inr{\dt{v}^\star, \widetilde{B}(\dt{u}^{\text{sc}} ,\dt{v}^{\text{sc}})^\top \mathbf 1_n} \big),
\end{align*}
which yields 
\begin{align*}
\Psi_{\varepsilon, \kappa} (u^{\text{sc}} ,v^{\text{sc}}) -\Psi(u^\star, v^\star)
&\leq  \inr{\kappa \dt{u}^{\text{sc}} - \dt{u}^{\star}, \widetilde{B}(\dt{u}^{\text{sc}} ,\dt{v}^{\text{sc}})\mathbf 1_m - \dt{\mu}} +\inr{\kappa^{-1}\dt{v}^{\text{sc}} - \dt{v}^{\star}, \widetilde{B}(\dt{u}^{\text{sc}} ,\dt{v}^{\text{sc}})^\top\mathbf 1_n - \dt{\nu}}\\
&\qquad + (1- \kappa)\inr{\dt{u}^{\text{sc}} ,\widetilde{B}(\dt{u}^{\text{sc}} ,\dt{v}^{\text{sc}})\mathbf 1_m} + (1- \kappa^{-1})\inr{\dt{v}^{\text{sc}} ,\widetilde{B}(\dt{u}^{\text{sc}} ,\dt{v}^{\text{sc}})^\top\mathbf 1_n}.
\end{align*}
Applying Holder's inequality gives 
\begin{align*}
\Psi_{\varepsilon, \kappa}(u^{\text{sc}} ,v^{\text{sc}}) -\Psi(u^\star, v^\star)
&\leq \norm{\kappa \dt{u}^{\text{sc}} - \dt{u}^\star}_\infty \norm{\widetilde{B}(\dt{u}^{\text{sc}} ,\dt{v}^{\text{sc}})\mathbf 1_m - \dt{\mu}}_1 + \norm{\kappa^{-1}\dt{v}^{\text{sc}} - \dt{v}^\star}_\infty \norm{\widetilde{B}(\dt{u}^{\text{sc}} ,\dt{v}^{\text{sc}})^\top\mathbf 1_n - \dt{\nu}}_1\\
&\qquad + |1- \kappa|\inr{\dt{u}^{\text{sc}} ,\widetilde{B}(\dt{u}^{\text{sc}} ,\dt{v}^{\text{sc}})\mathbf 1_m} + |1- \kappa^{-1}|\inr{\dt{v}^{\text{sc}} ,\widetilde{B}(\dt{u}^{\text{sc}} ,\dt{v}^{\text{sc}})^\top\mathbf 1_n}\\
&\leq \big(\norm{\dt{u}^{\text{sc}} - \dt{u}^\star}_\infty + |1 - \kappa| \norm{\dt{u}^{\text{sc}}}_\infty\big)\norm{\widetilde{B}(\dt{u}^{\text{sc}} ,\dt{v}^{\text{sc}})\mathbf 1_m - \dt{\mu}}_1\\
&\qquad + \big(\norm{\dt{v}^{\text{sc}} - \dt{v}^\star}_\infty + |1 - \kappa^{-1}| \norm{\dt{v}^{\text{sc}}}_\infty\big) \norm{\widetilde{B}(\dt{u}^{\text{sc}} ,\dt{v}^{\text{sc}})^\top\mathbf 1_n - \dt{\nu}}_1\\
&\qquad + |1- \kappa|\inr{\dt{u}^{\text{sc}} ,\widetilde{B}(\dt{u}^{\text{sc}} ,\dt{v}^{\text{sc}})\mathbf 1_m} + |1- \kappa^{-1}| \inr{\dt{v}^{\text{sc}} ,\widetilde{B}(\dt{u}^{\text{sc}} ,\dt{v}^{\text{sc}})^\top\mathbf 1_n}
\end{align*}
where, in the last inequality, we use the facts that $\norm{\kappa\dt{u}^{\text{sc}} - \dt{u}^\star}_\infty \leq \norm{\dt{u}^{\text{sc}} - \dt{u}^\star}_\infty + |1 - \kappa| \norm{\dt{u}^{\text{sc}}}_\infty$ and $\norm{\kappa^{-1}\dt{v}^{\text{sc}} - \dt{v}^\star}_\infty \leq \norm{\dt{v}^{\text{sc}} - \dt{v}^\star}_\infty + |1 - \kappa^{-1}| \norm{\dt{v}^{\text{sc}}}_\infty.$
Moreover, note that 
\begin{align*}
\begin{cases}
\norm{\dt{u}^{\text{sc}} - \dt{u}^\star}_\infty = \norm{{u}^{\text{sc}} - {u}^\star}_\infty,\\
\norm{\dt{v}^{\text{sc}} - \dt{v}^\star}_\infty = \norm{{v}^{\text{sc}} - {v}^\star}_\infty,
\end{cases}
\text{ and }
\begin{cases}
\norm{\widetilde{B}(\dt{u}^{\text{sc}} ,\dt{v}^{\text{sc}})\mathbf 1_m - \dt{\mu}}_1 = \norm{{B}({u}^{\text{sc}} ,{v}^{\text{sc}})\mathbf 1_m - {\mu}}_1 = \norm{\mu^{\text{sc}} - \mu}_1,\\
\norm{\widetilde{B}(\dt{u}^{\text{sc}} ,\dt{v}^{\text{sc}})^\top\mathbf 1_n - \dt{\nu}}_1 = \norm{{B}({u}^{\text{sc}} ,{v}^{\text{sc}})^\top\mathbf 1_n - {\nu}}_1 = \norm{\nu^{\text{sc}} - \nu}_1.
\end{cases}
\end{align*}
Then
\begin{align}
\Psi_{\varepsilon, \kappa}(u^{\text{sc}} ,v^{\text{sc}}) -\Psi(u^\star, v^\star) &
\leq \big(\norm{{u}^{\text{sc}} - {u}^\star}_\infty + |1- \kappa| \norm{{u}^{\text{sc}}}_\infty\big)\norm{\mu^{\text{sc}} - \mu}_1\nonumber\\
&\qquad + \big(\norm{{v}^{\text{sc}} - {v}^\star}_\infty + |1-\kappa^{-1}| \norm{{v}^{\text{sc}}}_\infty\big) \norm{\nu^{\text{sc}} - \nu}_1\nonumber\\
&\qquad + |1- \kappa|\inr{{u}^{\text{sc}} ,{\mu}^{\text{sc}}} + |1-\kappa^{-1}|\inr{{v}^{\text{sc}} ,{\nu}^{\text{sc}}}\nonumber\\
&\leq \big(\norm{{u}^{\text{sc}} - {u}^\star}_\infty + |1- \kappa| \norm{{u}^{\text{sc}}}_\infty\big)\norm{\mu^{\text{sc}} - \mu}_1\nonumber\\
&\qquad + \big(\norm{{v}^{\text{sc}} - {v}^\star}_\infty + |1-\kappa^{-1}| \norm{{v}^{\text{sc}}}_\infty\big) \norm{\nu^{\text{sc}} - \nu}_1 \label{priori_bounds_induction}\\
&\qquad + |1- \kappa|\norm{{u}^{\text{sc}}}_\infty\norm{{\mu}^{\text{sc}}}_1 + |1- \kappa^{-1}|\norm{{v}^{\text{sc}}}_\infty\norm{{\nu}^{\text{sc}}}_1.\nonumber
\end{align}
Next, we bound the two terms $\norm{{u}^{\textrm{sc}} - {u}^\star}_\infty $ and $\norm{{v}^{\text{sc}} - {v}^\star}_\infty.$
If $r \in I^\complement_{\varepsilon, \kappa}$, then we have 
\begin{align*}
|({u}^{\text{sc}})_{r} - {u}^\star_r| 
&= \bigg|\log\bigg(\frac{\sum_{j=1}^m K_{rj}e^{v^\star_j}}{\sum_{j=1}^m \frac{\kappa \mu_r}{m\varepsilon}}\bigg) \bigg|\\
& \overset{(\star)}{\leq} \bigg|\log\bigg(\max_{1\leq i \leq m} \frac{K_{rj}e^{v^\star_j}}{\frac{\kappa \mu_r}{m\varepsilon}}\bigg) \bigg|\\
& \leq \big|\max_{1\leq j \leq m} (v_j^\star - \log(\frac{\kappa \mu_r}{m\varepsilon})\big|\\
& \leq \norm{v^\star - \log(\frac{\kappa \mu_r}{m\varepsilon})}_\infty \\
& \leq \norm{v^\star - {v}^{\text{sc}}}_\infty + \log(\frac{m\varepsilon^2}{c_{\mu\nu}}).
\end{align*}
where the inequality $(\star)$ comes from the fact that $\frac{\sum_{j=1}^na_j}{\sum_{j=1}^nb_{j}} \leq \max_{1\leq j \leq n} \frac{a_j}{b_j}, \forall a_j, b_j>0.$
Now, if $r \in I_{\varepsilon, \kappa}$, we get 
\begin{align*}
|{u}^{\text{sc}}_{r} - {u}^\star_r|
= \bigg|\log\bigg(\frac{\kappa\sum_{j=1}^m K_{rj}e^{v^\star_j}}{\sum_{j=1}^m K_{rj}e^{(v^{\text{sc}})_j}}\bigg)\bigg|
\leq \bigg|\log\bigg(\frac{\sum_{j=1}^m K_{rj}e^{v^\star_j}}{\sum_{j=1}^m K_{rj}e^{(v^{\text{sc}})_j}}\bigg)\bigg|
\overset{(\star)}{\leq}\norm{v^{\text{sc}} - v^\star}_\infty.
\end{align*}
If $s \in J^\complement_{\varepsilon, \kappa}$ then 
\begin{align*}
|{v}^{\text{sc}}_{s} - {v}^\star_s| &= \bigg| \log(\varepsilon\kappa) - \log(\frac{\nu_s}{\sum_{i=1}^n K_{is}e^{u^\star_i}})\bigg|\\
& \leq \bigg|\log\bigg(\max_{1\leq i \leq n} \frac{K_{is}e^{u^\star_i}}{\frac{\nu_s}{n\kappa\varepsilon}}\bigg) \bigg|\\
& \overset{(\star)}{\leq} \big|\max_{1\leq i \leq n} (u_i^\star - \log(\frac{\nu_s}{n\kappa\varepsilon})\big|\\
& \leq \norm{u^\star - \log(\frac{\nu_s}{n\kappa\varepsilon})}_\infty \\
& \leq \norm{u^\star - u^{\text{sc}}}_\infty +  \log(\frac{n\varepsilon^2}{c_{\mu\nu}}).
\end{align*}
If $s \in J_{\varepsilon, \kappa}$ then 
\begin{align*}
|{v}^{\text{sc}}_{s} - {v}^\star_s| 
= \bigg|\log\bigg(\frac{\kappa\sum_{i=1}^m K_{ri}e^{u^\star_i}}{\sum_{i=1}^m K_{ri}e^{(u^{\text{sc}})_i}}\bigg)\bigg|
 \leq \bigg|\log\bigg(\frac{\kappa\sum_{i=1}^m K_{ri}e^{v^\star_j}}{\sum_{i=1}^m K_{ri}e^{(u^{\text{sc}})_i}}\bigg)\bigg|
\overset{(\star)}{\leq} \norm{u^{\text{sc}} - u^\star}_\infty.
\end{align*}
Therefore, we obtain the followoing bound:
\begin{align}
\label{induction_uppers_bounds}
\max\{\norm{u^\star - u^{\text{sc}}}_\infty, \norm{v^\star - v^{\text{sc}}}_\infty\} &\leq \max\Big\{\norm{u^\star}_\infty + \norm{u^{\text{sc}}}_\infty +  \log(\frac{n\varepsilon^2}{c_{\mu\nu}}), \norm{v^\star}_\infty + \norm{v^{\text{sc}}}_\infty + \log(\frac{m\varepsilon^2}{c_{\mu\nu}})\Big\}\nonumber\\
&\leq 2\Big(\norm{u^\star}_\infty + \norm{v^\star}_\infty +  \norm{u^{\text{sc}}}_\infty +  \norm{v^{\text{sc}}}_\infty + \log\big(\frac{(n\vee m)\varepsilon^2}{c_{\mu\nu}}\big)\Big).
\end{align}
Now, Lemma 3.2 in~\cite{lin2019} provides an upper bound for the $\ell_\infty$ of the optimal solution pair $(u^\star, v^\star)$ of problem~\eqref{sinkhorn-dual} as follows:
$\norm{u^\star}_\infty \leq A$ and $\norm{v^\star}_\infty\leq A,$ where 
\begin{align}
\label{R_constant}
A = \frac{\norm{C}_\infty}{\eta} + \log\big(\frac{n\vee m}{c^2_{\mu\nu}}\big).
\end{align}
Plugging~\eqref{induction_uppers_bounds} and~\eqref{R_constant} in~\eqref{priori_bounds_induction}, we obtain
\begin{align}
\Psi_{\varepsilon, \kappa}(u^{\text{sc}} ,v^{\text{sc}}) -\Psi(u^\star, v^\star) 
&\leq 2\big(A + \norm{u^{\text{sc}}}_\infty +  \norm{v^{\text{sc}}}_\infty + \log\big(\frac{(n\vee m)\varepsilon^2}{c_{\mu\nu}})\big) \big(\norm{\mu^{\text{sc}} - \mu}_1+ \norm{\nu^{\text{sc}} - \nu}_1\big)\nonumber\\
&\qquad + |1- \kappa|\big(\norm{{u}^{\text{sc}}}_\infty\norm{{\mu}^{\text{sc}}}_1 +\norm{\mu^{\text{sc}} - \mu}_1\big)\label{post_bounds_induction_2}\\
&\qquad  + |1-\kappa^{-1}|\big(\norm{{v}^{\text{sc}}}_\infty\norm{{\nu}^{\text{sc}}}_1+\norm{\nu^{\text{sc}} - \nu}_1\big)\nonumber.
\end{align}
By Proposition~\ref{prop:bounds_of_usc_and_vsc}, we have 
\begin{align*} 
\norm{u^{\text{sc}}}_\infty \leq \log\big(\frac \varepsilon\kappa\vee \frac{1}{m\varepsilon K_{\min}}\big) \text{ and } \norm{v^{\text{sc}}}_\infty \leq \log\big(\varepsilon\kappa \vee \frac{1}{n\varepsilon K_{\min} }\big)
\end{align*}
and hence by Remark~\ref{rem:orders_of_epsilonappa}, 
\begin{equation*}
\norm{u^{\text{sc}}}_\infty = \bigO\big(\log({n^{1/4}}/{(mK_{\min})^{3/4}c_{\mu\nu}^{1/4}})\big) \text{ and }\norm{u^{\text{sc}}}_\infty =\bigO\big(\log({m^{1/4}}/{(nK_{\min})^{3/4}c_{\mu\nu}^{1/4}})\big).
\end{equation*}
Acknowledging that $\log(1/ K_{\min}^2) = 2\norm{C}_\infty/\eta$, we have 
\begin{align*}
A + \norm{u^{\text{sc}}}_\infty  + \norm{v^{\text{sc}}}_\infty + \log\big(\frac{(n\vee m)\varepsilon^2}{c_{\mu\nu}})\big) = \bigO\Big(\frac{\norm{C}_\infty}{\eta} + \log\Big(\frac{(n\vee m)^2}{nmc_{\mu\nu}^{7/2}}\Big)\Big).
\end{align*}
Letting $\Omega_{\kappa} := |1- \kappa|\big(\norm{{u}^{\text{sc}}}_\infty\norm{{\mu}^{\text{sc}}}_1 +\norm{\mu^{\text{sc}} - \mu}_1\big) + |1-\kappa^{-1}| \big(\norm{{v}^{\text{sc}}}_\infty\norm{{\nu}^{\text{sc}}}_1+\norm{\nu^{\text{sc}} - \nu}_1\big).$
We have that 
\begin{align*}
\Omega_{\kappa} &=\bigO\Big(\Big(\frac{\norm{C}_\infty}{\eta} + \log\big(\frac{1}{(nm)^{3/4} c_{\mu\nu}^{1/2}}\big)\Big)\Big(|1- \kappa|(\norm{{\mu}^{\text{sc}}}_1 +\norm{\mu^{\text{sc}} - \mu}_1\big) + |1-\kappa^{-1}|(\norm{{\nu}^{\text{sc}}}_1+\norm{\nu^{\text{sc}} - \nu}_1)\Big)\Big)\\
&= \bigO\Big(\Big(\frac{\norm{C}_\infty}{\eta} + \log\Big(\frac{(n\vee m)^2}{nmc_{\mu\nu}^{7/2}}\Big)\Big)\big(|1- \kappa|\norm{\mu^{\text{sc}}}_1 + |1-\kappa^{-1}|\norm{\nu^{\text{sc}}}_1 + |1- \kappa| + |1 - \kappa^{-1}|\big)\Big).
\end{align*}
Hence, we arrive at 
\begin{align*}
\Psi_{\varepsilon, \kappa}(u^{\text{sc}} ,v^{\text{sc}}) -\Psi(u^\star, v^\star) 
& = \bigO\big(R(\norm{\mu - {\mu}^{\text{sc}}}_1 + \norm{\nu - \nu^{\text{sc}}}_1 + \omega_{\kappa})\big).
\end{align*}
$\hfill\square$

To more characterize $\omega_\kappa$, the following 
lemma expresses an upper bound with respect to $\ell_1$-norm of $\mu^{\text{sc}}$and $\nu^{\text{sc}}$. 

\begin{lemma}
\label{lemma_bounds_on_marginals}
Let $(u^{\text{sc}}, v^{\text{sc}})$ be an optimal solution of problem~\eqref{screen-sinkhorn_second_def}.
Then one has 
\begin{align}
\label{l1-norm-mu-sc}
\norm{\mu^{\text{sc}}}_1 = \bigO\Big(\frac{n_b\sqrt{m}}{\sqrt{nK_{\min}c_{\mu\nu}}} + (n-n_b)\Big(\frac{m_b}{\sqrt{nmc_{\mu\nu}}K_{\min}^{3/2}} + \frac{m-m_b}{\sqrt{nm}K_{\min}}\Big)\Big),
\end{align}
and
\begin{align}
\label{l1-norm-nu-sc}
\norm{\nu^{\text{sc}}}_1 = \bigO\Big(\frac{m_b\sqrt{n}}{\sqrt{mK_{\min}c_{\mu\nu}}} + (m-m_b)\Big(\frac{n_b}{\sqrt{nmc_{\mu\nu}}K_{\min}^{3/2}} + \frac{n-n_b}{\sqrt{nm}K_{\min}}\Big)\Big).
\end{align}
\end{lemma}

\begin{proof}
Using inequality~\eqref{bound_on_v}, we obtain 
\begin{align*}
\norm{\mu^{\text{sc}}}_1 &= \sum_{i \in I_{\varepsilon,\kappa}} \mu^{\text{sc}}_i +  \sum_{i \in I^\complement_{\varepsilon,\kappa}}\mu^{\text{sc}}_i\\
& \overset{\eqref{i-th-marginal-mu}}{=} \kappa \norm{\mu_{I_{\varepsilon,\kappa}}^{\text{sc}}}_1 + \frac \varepsilon\kappa \sum_{i \in I^\complement}\Big( \sum_{j \in J_{\varepsilon,\kappa}} K_{ij} e^{v^{\text{sc}}_j} + \varepsilon\kappa \sum_{j\in J^\complement_{\varepsilon,\kappa}}K_{ij}\Big)\\
& \overset{\eqref{bound_on_v}}{\leq} \kappa \norm{\mu_{I_{\varepsilon,\kappa}}^{\text{sc}}}_1 + (n-n_b) \Big(\frac{m_b\max_{j \in J_{\varepsilon,\kappa}} \nu_j}{n\kappa K_{\min}} + (m-m_b)\varepsilon^2 \Big).
\end{align*}
Using Remark~\ref{rem:orders_of_epsilonappa}, we get the desired closed form in~\eqref{l1-norm-mu-sc}.
Similarly, we can prove the same statement for $\norm{\nu^{\text{sc}}}_1$.
\end{proof}

{

\subsection{Additional experimental results}

\paragraph{Experimental setup.}

	All computations have been run on each single core of an Intel Xeon E5-2630 processor clocked at
	2.4 GHz in a Linux machine with 144 Gb of memory.

\paragraph{On the use of a constrained L-BFGS-B solver.}

It is worth to note that standard Sinkhorn's alternating minimization cannot be applied for the constrained screened dual problem~\eqref{screen-sinkhorn_second_def}. This appears more clearly while writing its optimality conditions (see Equations~\eqref{i-th-marginal-mu} and~\eqref{i-th-marginal-nu} ).
We resort to a L-BFGS-B algorithm to solve the constrained convex optimization problem on the screened variables~\eqref{screen-sinkhorn_second_def}, but any other efficient solver (e.g., proximal based method or Newton method) could be used. 
The choice of the starting point for the L-BFGS-B algorithm is given by the solution of the \textsc{Restricted Sinkhorn} method (see Algorithm~\ref{restricted_sinkhorn}), which is a Sinkhorn-like algorithm applied to
the active dual variables. While simple and efficient the solution of this 
\textsc{restricted Sinkhorn} algorithm does not satisfy the lower bound constraints of Problem \eqref{screen-sinkhorn_second_def}.
We further note that, as for the \textsc{Sinkhorn} algorithm, our \textsc{Screenkhorn} algorithm can be accelerated using a GPU implementation\footnote{\url{https://github.com/nepluno/lbfgsb-gpu}} of the L-BFGS-B algorithm~\citep{fei2014}.

\LinesNotNumbered
\begin{algorithm}[htbp]
\SetNlSty{textbf}{}{.}
\DontPrintSemicolon
\caption{\textsc{Restricted \textsc{Sinkhorn}}} \label{restricted_sinkhorn}
\nl \textbf{set:} $\bar{f}_u = \varepsilon\kappa\, c(K_{I_{\varepsilon, \kappa}, J^\complement_{\varepsilon, \kappa}}), \bar{f}_v  = \varepsilon\kappa^{-1} \,r(K_{I^\complement_{\varepsilon, \kappa}, J_{\varepsilon, \kappa}});$\\
\nl \For{$t=1,2,3$}{
	$f^{(t)}_v \gets K^\top_{I_{\varepsilon, \kappa}, J_{\varepsilon, \kappa}} u + \bar{f}_v;$\\
	$v^{(t)} \gets \frac{\nu_{J_{\varepsilon, \kappa}}}{\kappa f^{(t)}_v} ;$\\
	$f^{(t)}_u \gets K_{I_{\varepsilon, \kappa}, J_{\varepsilon, \kappa}} v + \bar{f}_u;$\\
	$u^{(t)} \gets \frac{\kappa \mu_{I_{\varepsilon, \kappa}}}{f^{(t)}_u};$\\
	$u \gets u^{(t)}$, $v \gets v^{(t)}$;
}
\nl \Return{$(u^{(t)}, v^{(t)})$}
\end{algorithm}

\paragraph{Comparison with other solvers.}

We have considered experiments with \textsc{Greenkhorn} algorithm~\citep{altschulernips17} but the implementation in POT library and our Python version of Matlab Altschuler's Greenkhorn code\footnote{\url{https://github.com/JasonAltschuler/OptimalTransportNIPS17}} were not competitive with \textsc{Sinkhorn}. Hence, for both versions, \textsc{Screenkhorn} is more competitive than \textsc{Greenkhorn}. The computation time gain reaches an order of $30$ when comparing our method with \textsc{Greenkhorn} while \textsc{Screenkhorn} is almost $2$ times faster than $\textsc{Sinkhorn}$. 
}

\begin{figure*}[htbp]
	\centering
	\includegraphics[width=8cm,height=6cm]{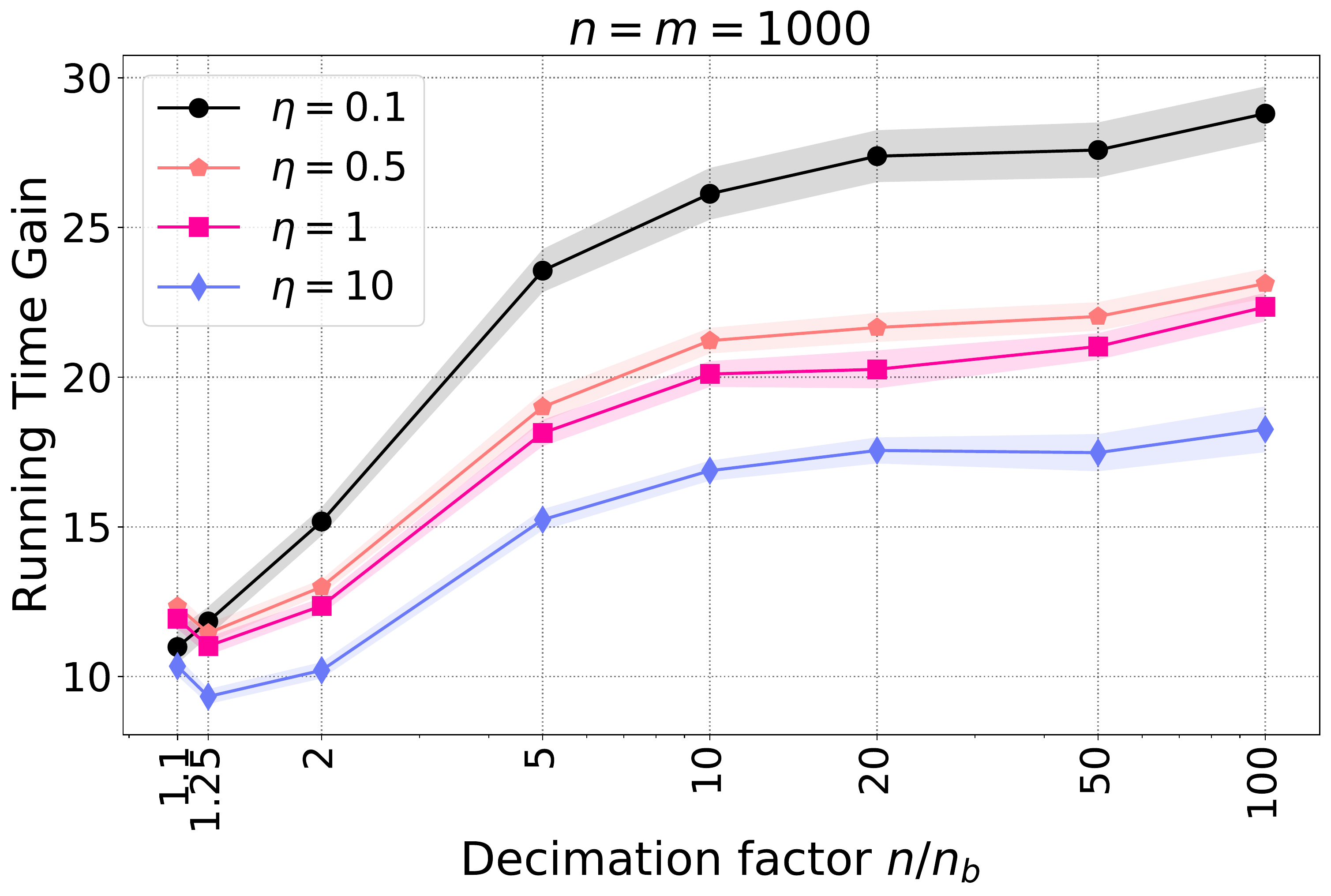}
	\caption{{$\frac{T_{\textsc{Greenkhorn}}}{T_{\text{\textsc{Screenkhorn}}}}$: Running time gain  for the toy problem (see Section~\ref{subsec:analysing_toy_problem}) as a function of the data decimation factor in \textsc{Screenkhorn}, for different settings of the regularization parameter $\eta$.
		\label{fig:screenkhorn_greenkhorn_gain}}}
\end{figure*}

\begin{figure*}[htbp]
	\centering
		~\hfill~\includegraphics[width=6.5cm]{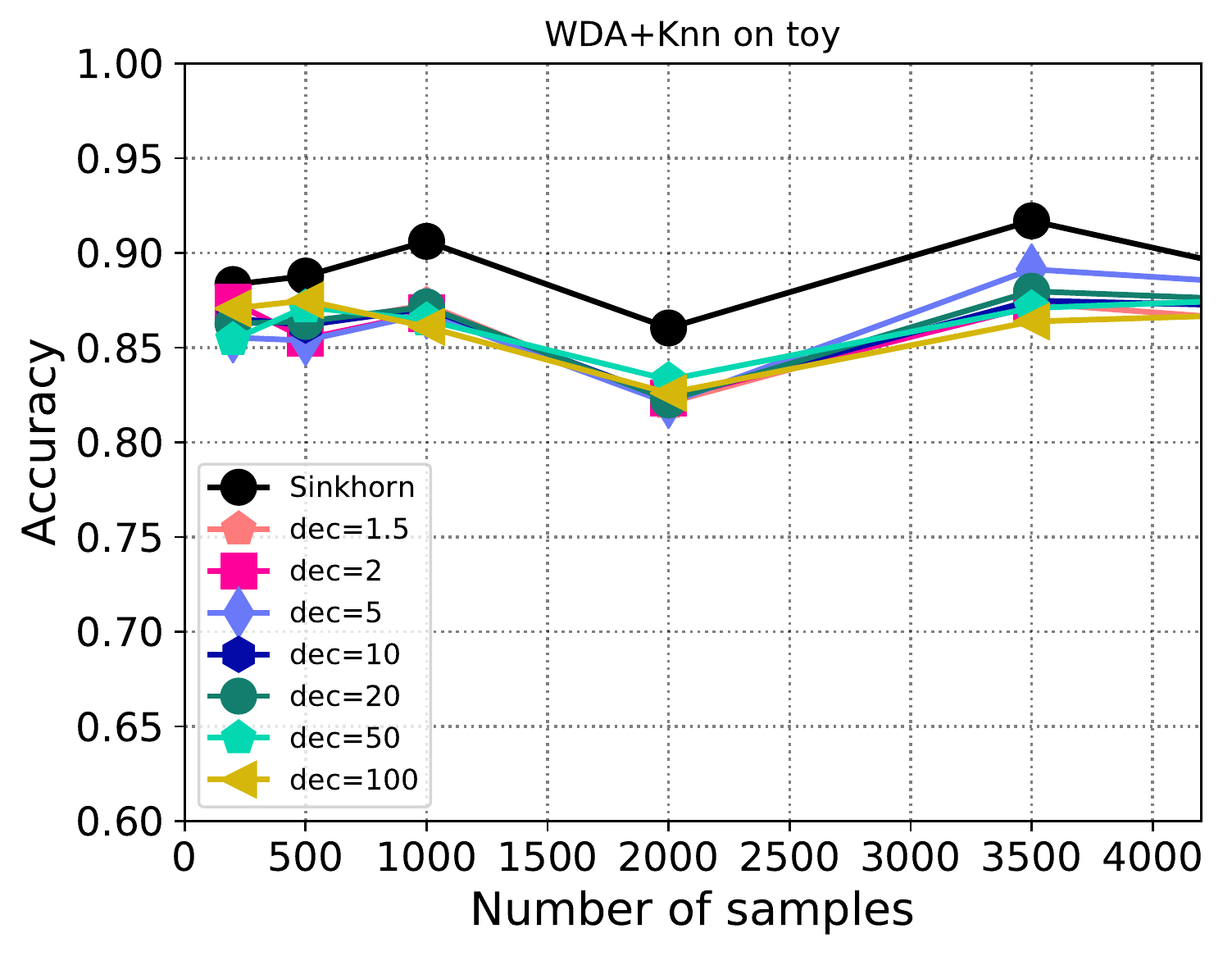}~\hfill~
	\includegraphics[width=6.5cm]{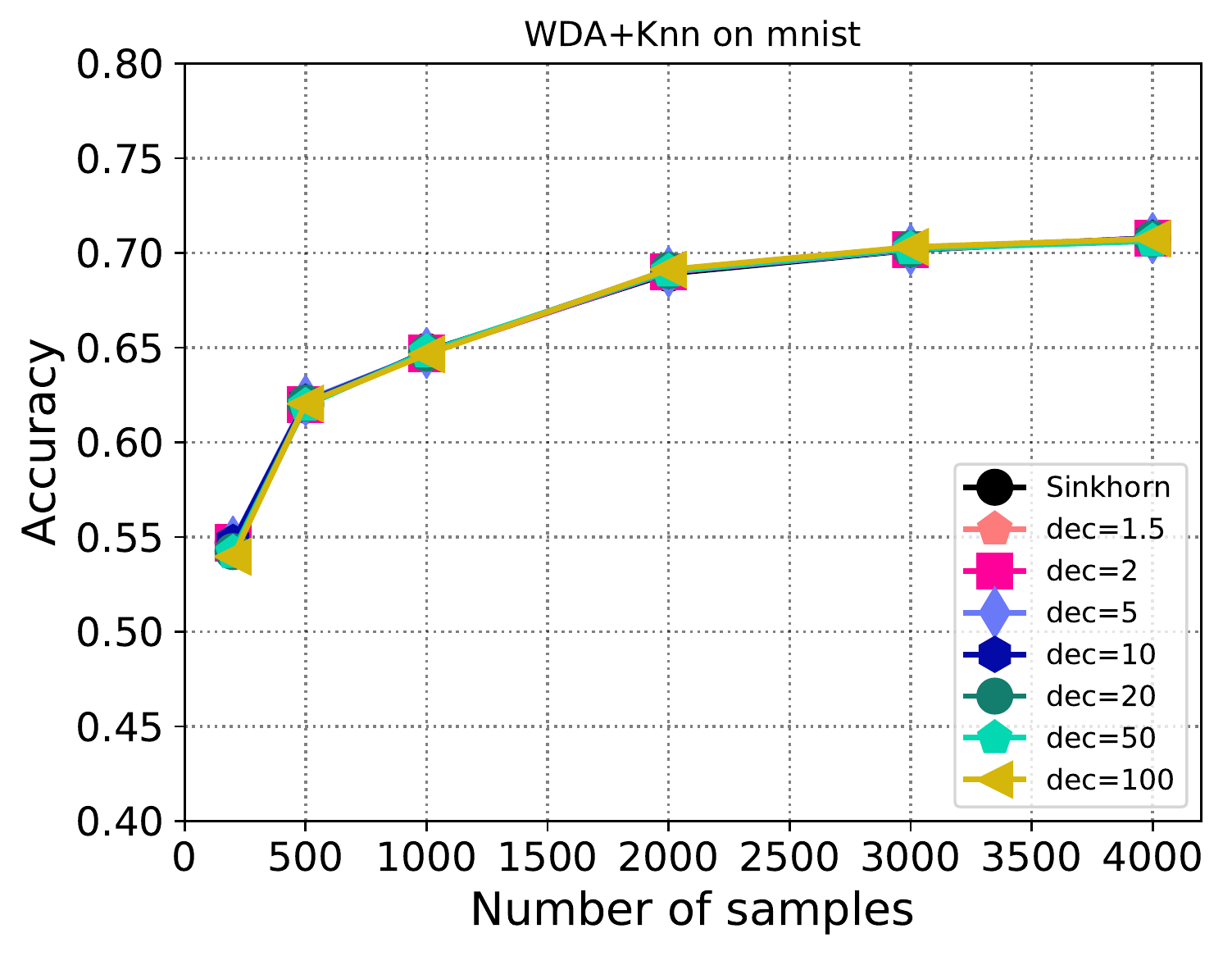}~\hfill~\\
	\caption{Accuracy of a $1$-nearest-neighbour after WDA for the (left) toy problem and, (right) MNIST). We note a slight loss of performance for the toy problem, whereas for MNIST, all approaches yield the same performance. 
		\label{fig:wda_gain}}
\end{figure*}

\begin{figure*}[htbp]

\begin{center}
	
	~\hfill~\includegraphics[width=6.6cm]{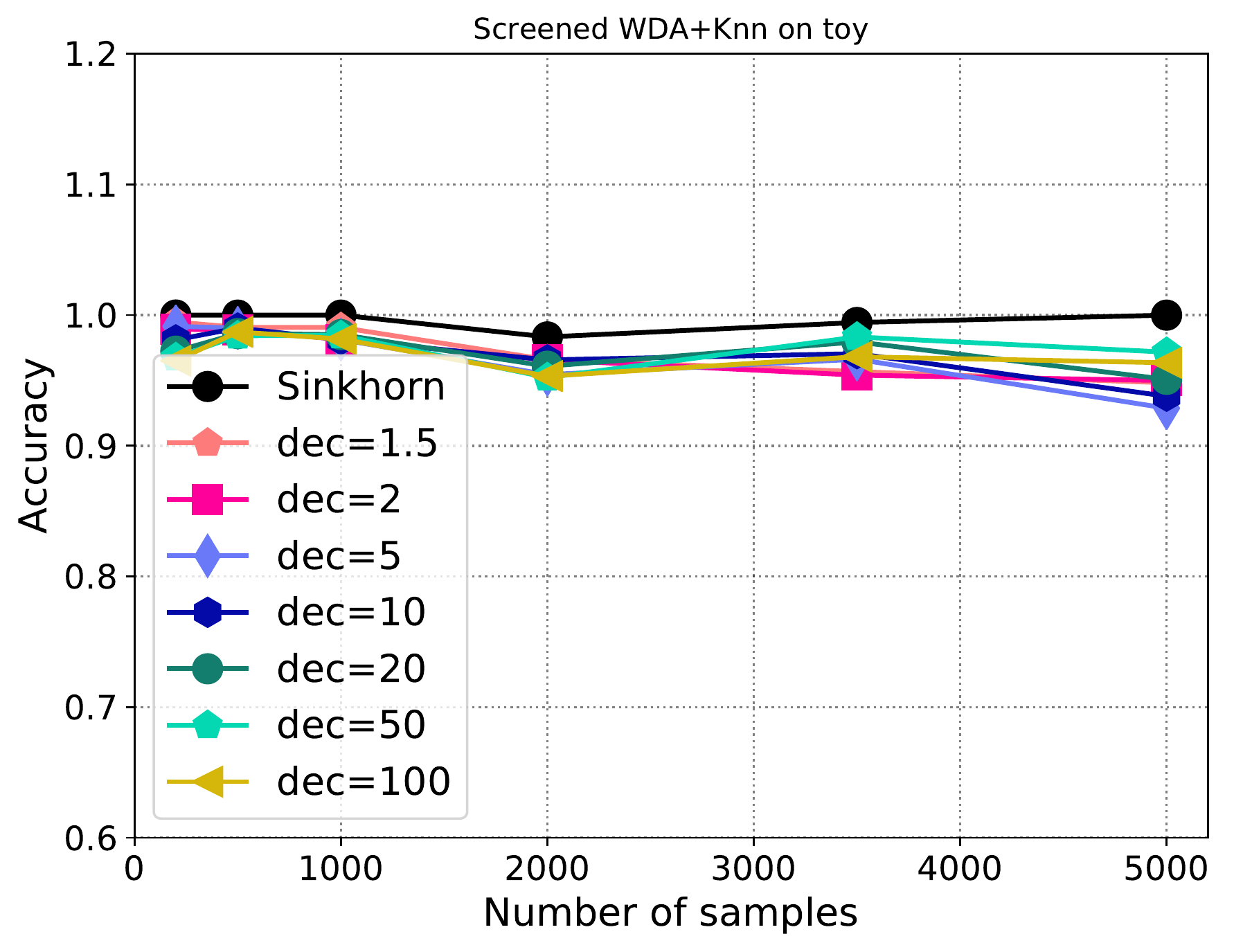}~\hfill~
	\includegraphics[width=6.5cm]{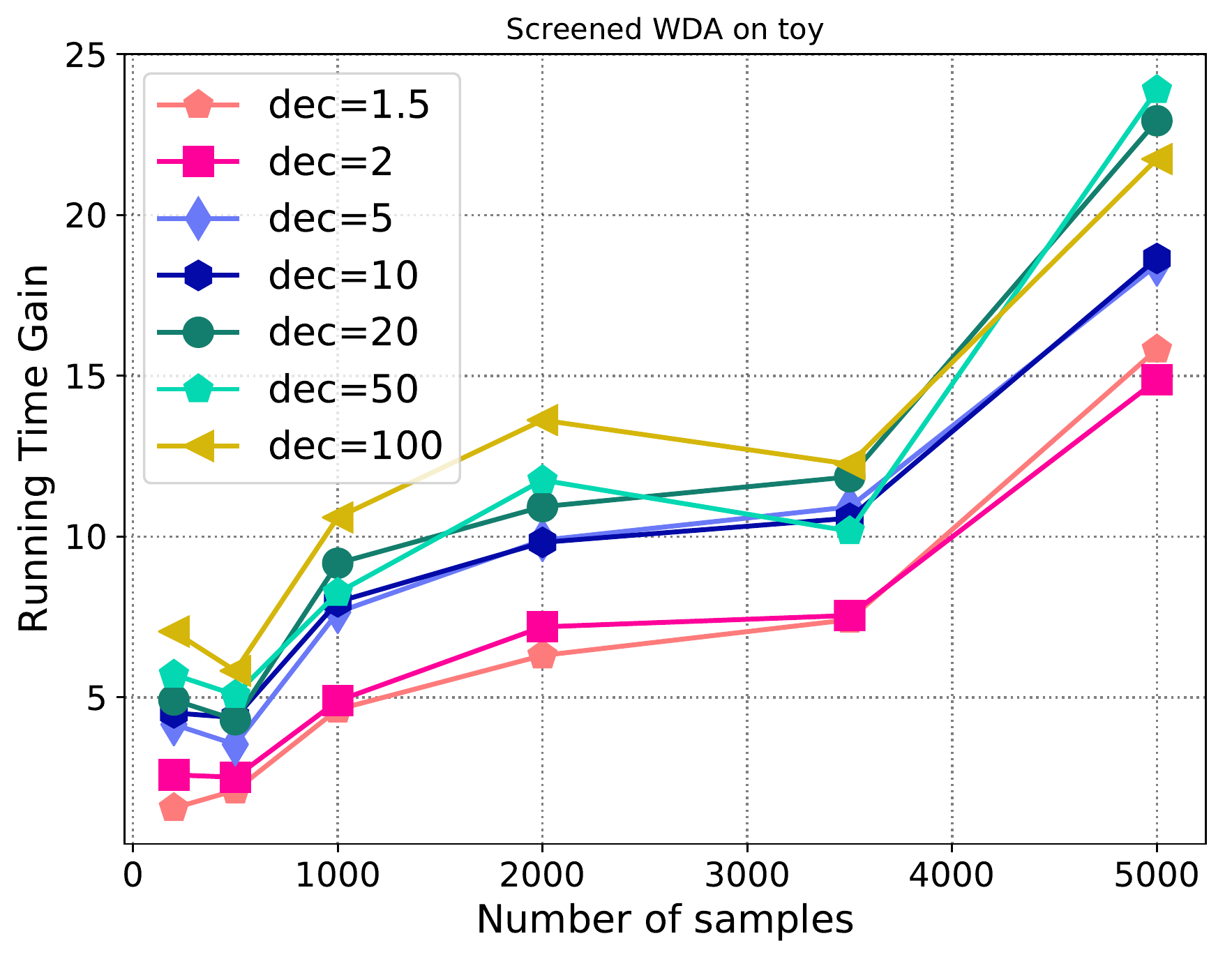}~\hfill~\\
	
	~\hfill~\includegraphics[width=6.5cm]{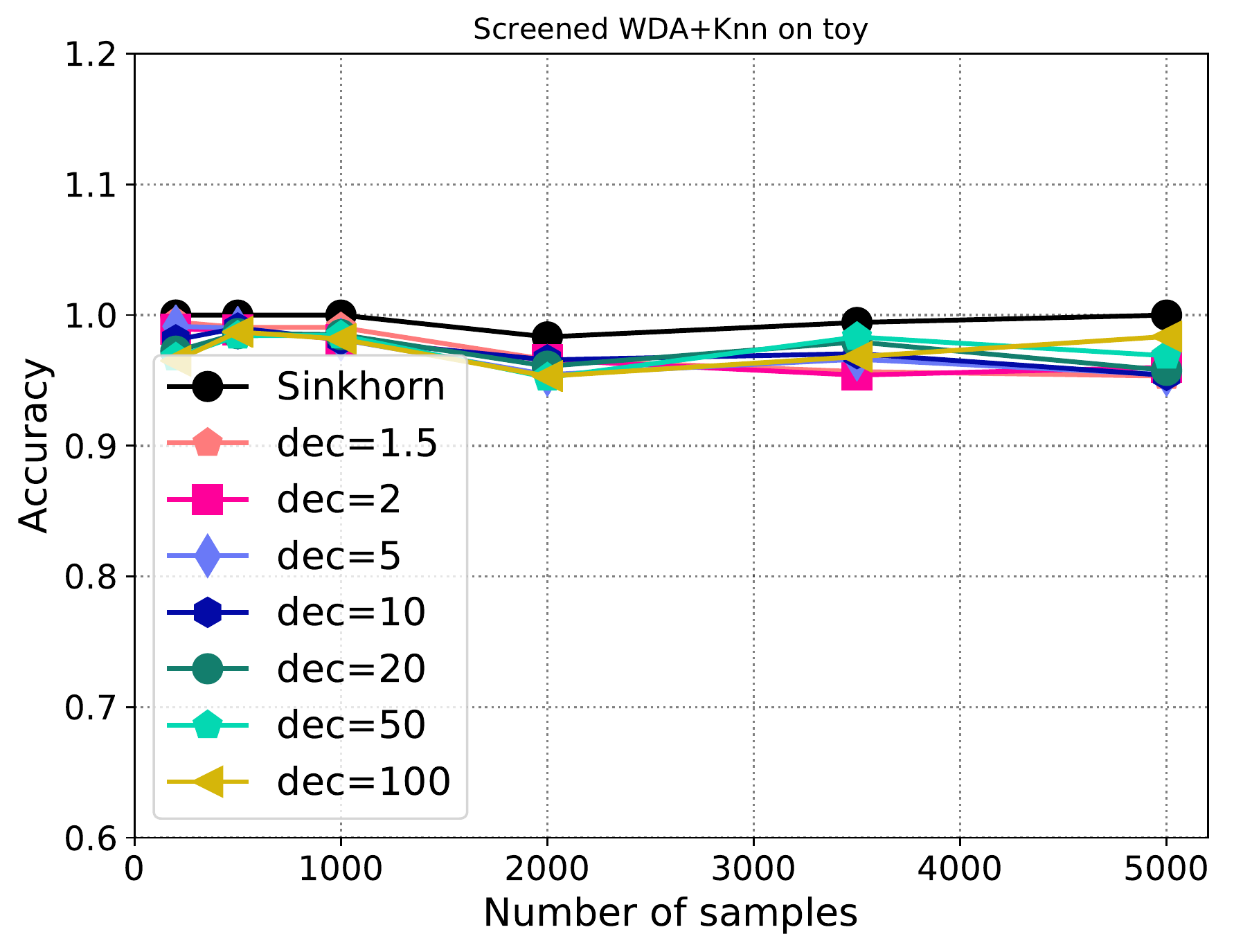}~\hfill~
	\includegraphics[width=6.5cm]{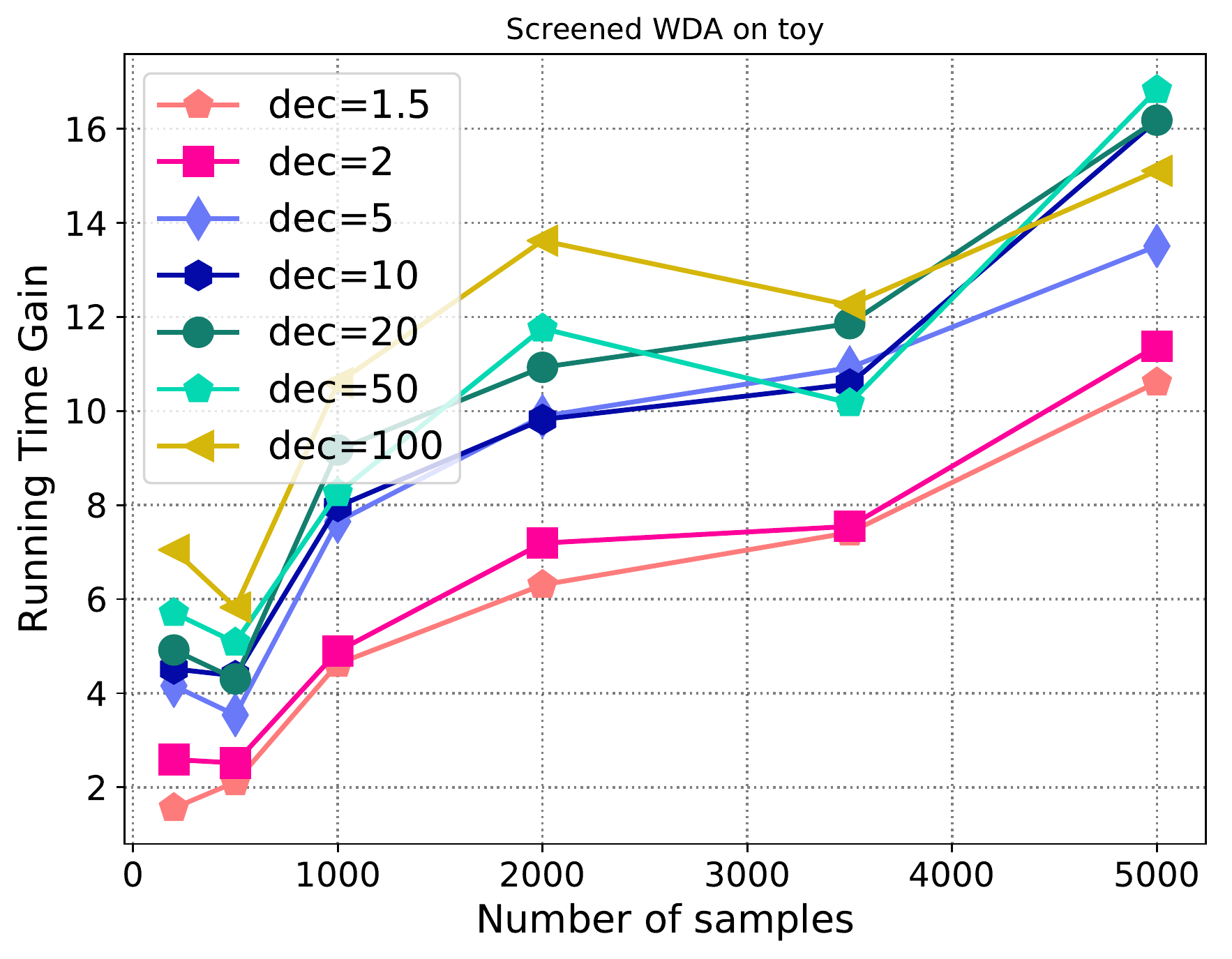}~\hfill~\\

	\caption{(top-left) Accuracy and (bottom-right) computational time gain on the toy dataset for $\eta=0.1$ and $1$-nearest-neighbour. (bottom) accuracy and gain but for a $5$-nearest-neighbour. We can note that a  slight loss of performances occur for larger training set sizes especially for $1$-nearest-neighbour. Computational gains increase with the dataset size and
		are on average of the order of magnitude. 
		\label{fig:wda_gain_toy_eta01}}

\end{center}

\end{figure*}

\begin{figure*}[htbp]
	\centering
		~\hfill~\includegraphics[width=6.cm]{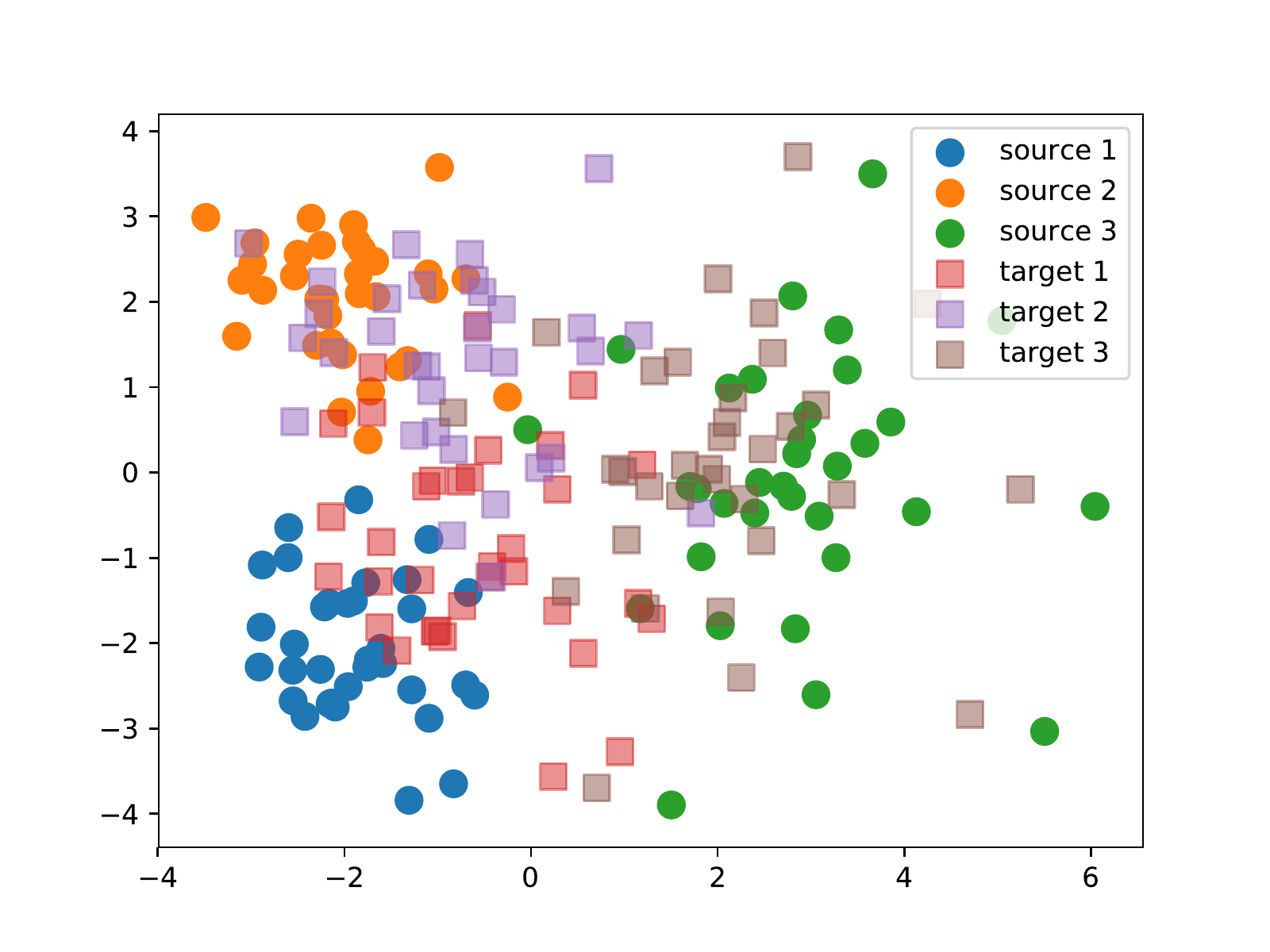}~\hfill~
	\includegraphics[width=6.cm]{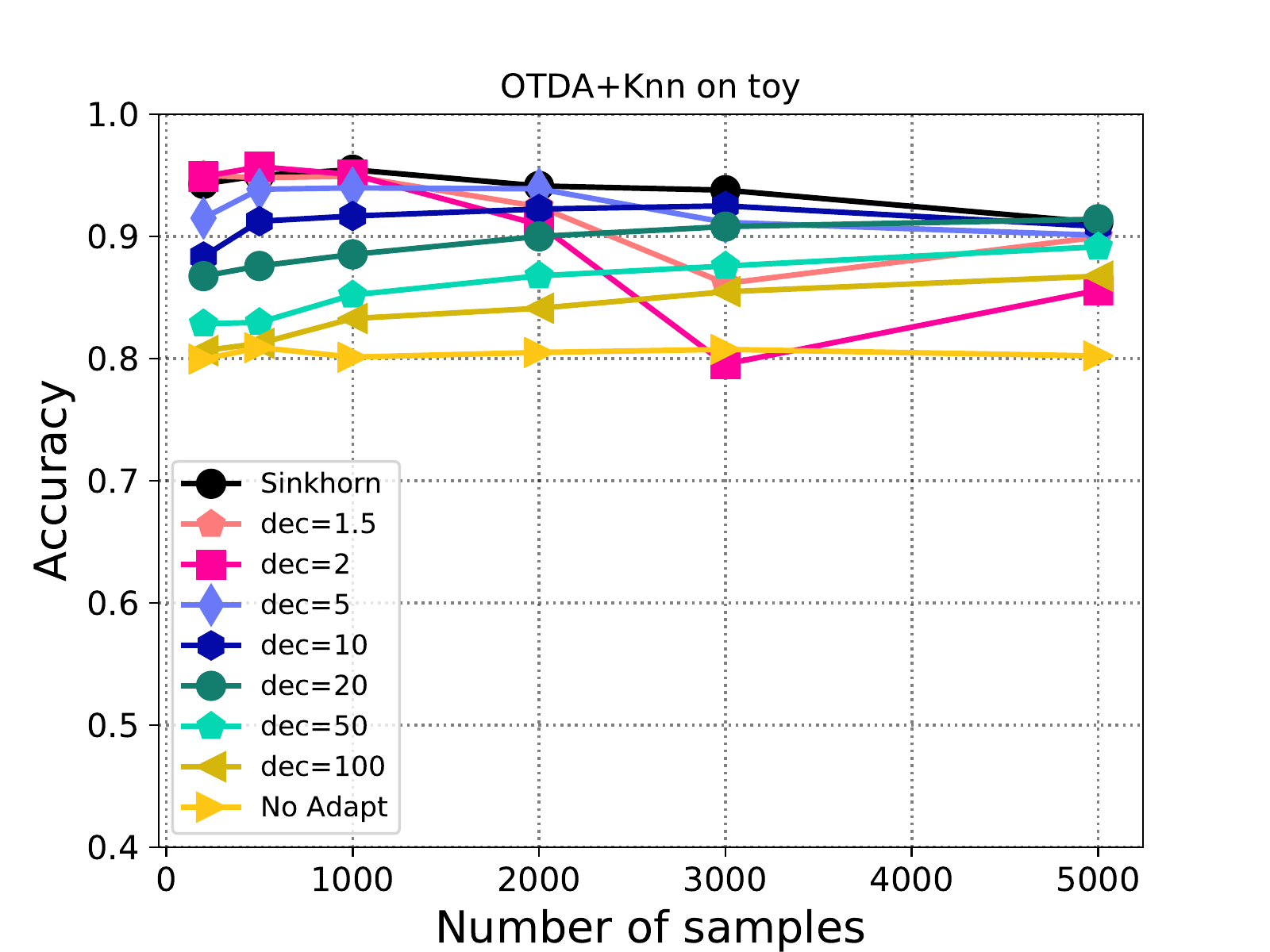}~\hfill~\\
	~\hfill~\includegraphics[width=6.cm]{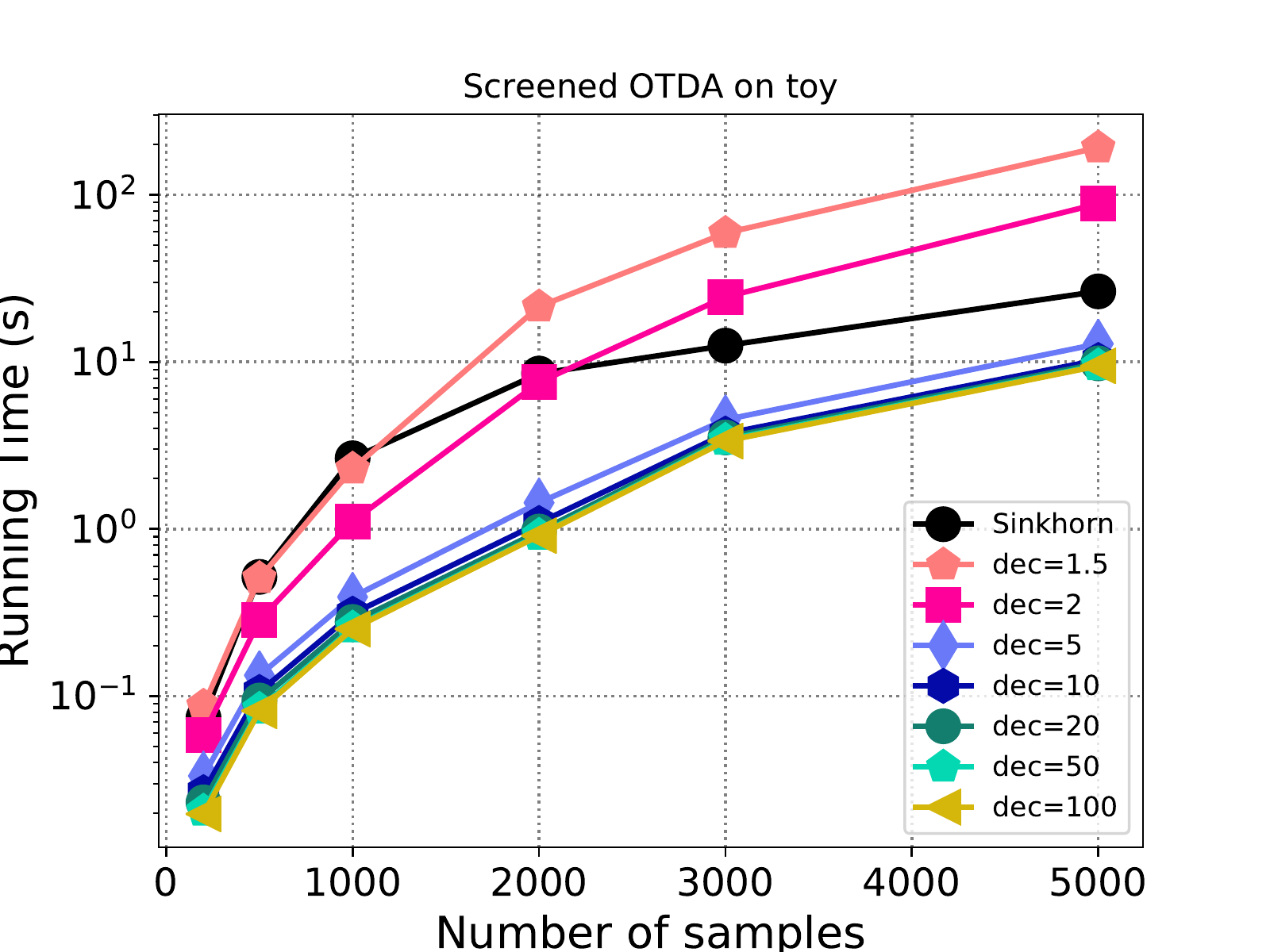}
	~\hfill~\includegraphics[width=6.cm]{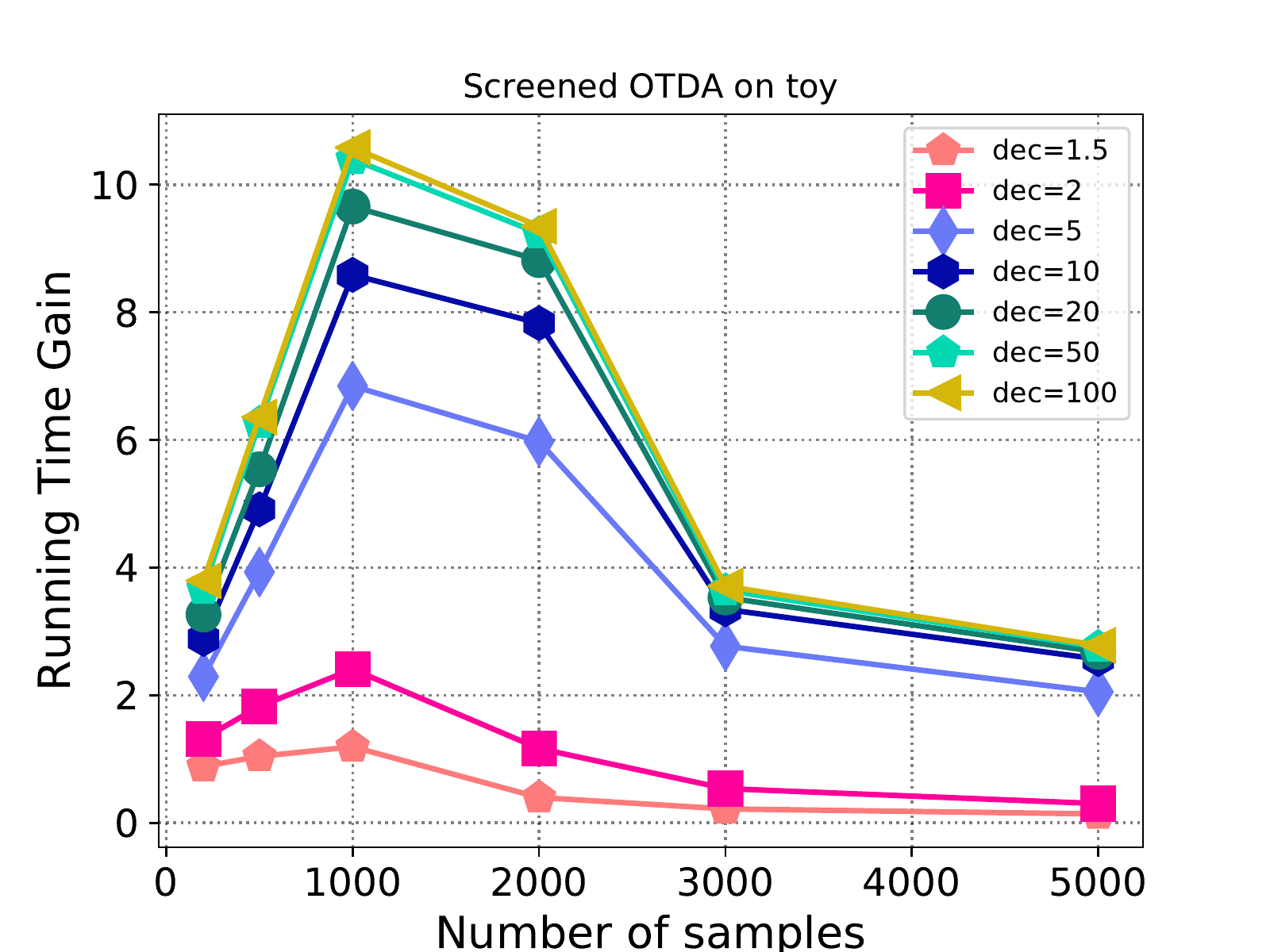}~\hfill~\\
	\caption{OT Domain Adaptation on a 3-class Gaussian toy problem. (top-left) Examples of source and target samples. (top-right) Evolution of the accuracy of a 1-nearest-neighbour classifier with respect to the number of samples. (bottom-left) Running time of the \textsc{Sinkhorn} and \textsc{Screenkhorn} for different decimation
	factors. (bottom-right). Gain in computation time. This toy problem is a problem in which classes are overlapping and distance between samples are rather limited. According to our analysis, this may be a situation in which \textsc{Screenkhorn} may
	result in smaller computational gain. We can remark that with respect to the accuracy
	\textsc{Screenkhorn} with decimation factors up to $10$ are competitive with \textsc{Sinkhorn}, although a slight loss of performance. Regarding computational time, for this example, small decimation factors does not  result in gain. However for above  $5$-factor decimation, the gain goes from  $2$ to $10$ depending on
	the number of samples. 
	\label{fig:otda:extra}}
\end{figure*}

\begin{figure*}[htbp]
	\centering
	\includegraphics[width=6.5cm]{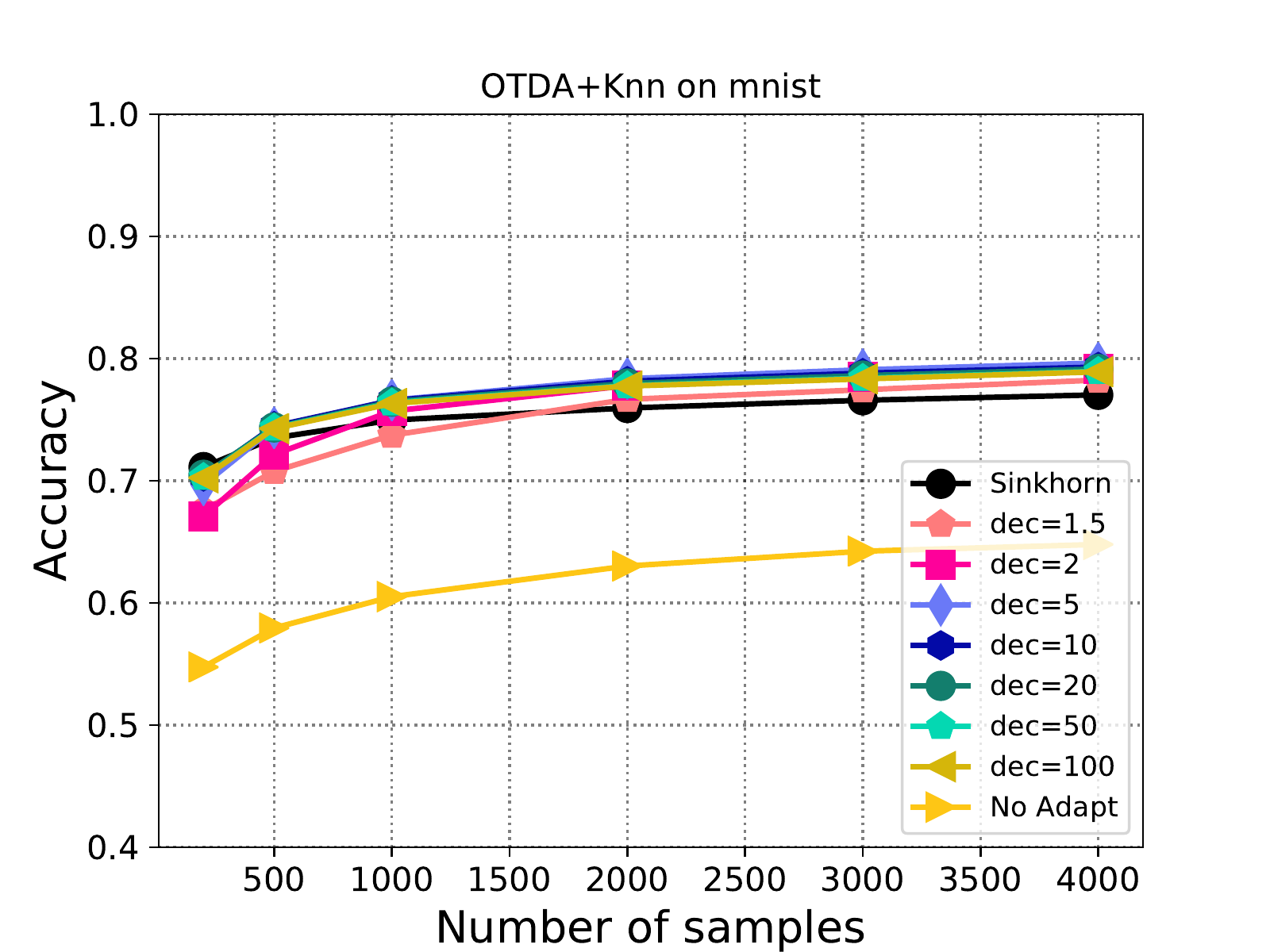}
	\includegraphics[width=6.5cm]{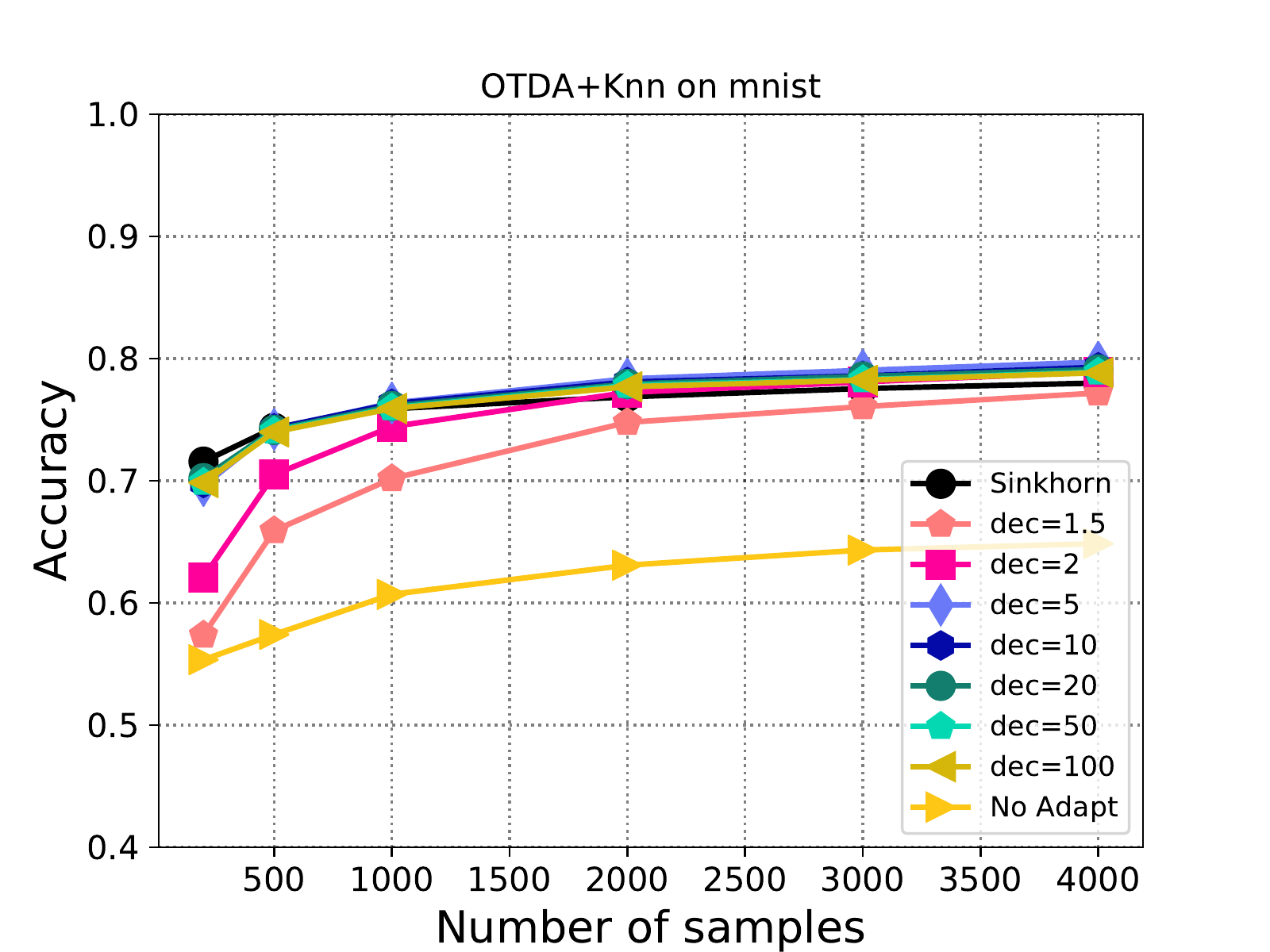}
	\includegraphics[width=6.5cm]{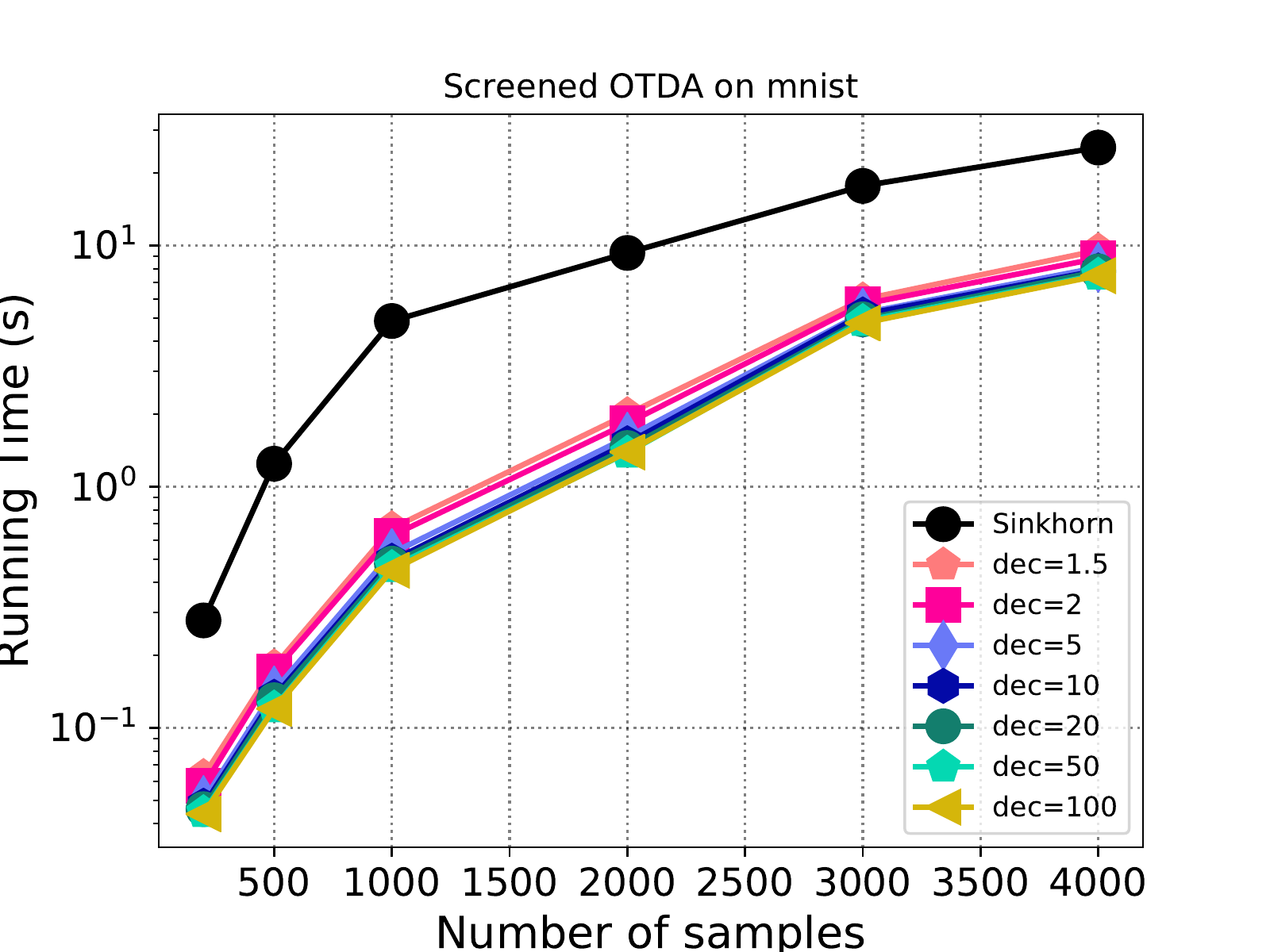}
	\includegraphics[width=6.5cm]{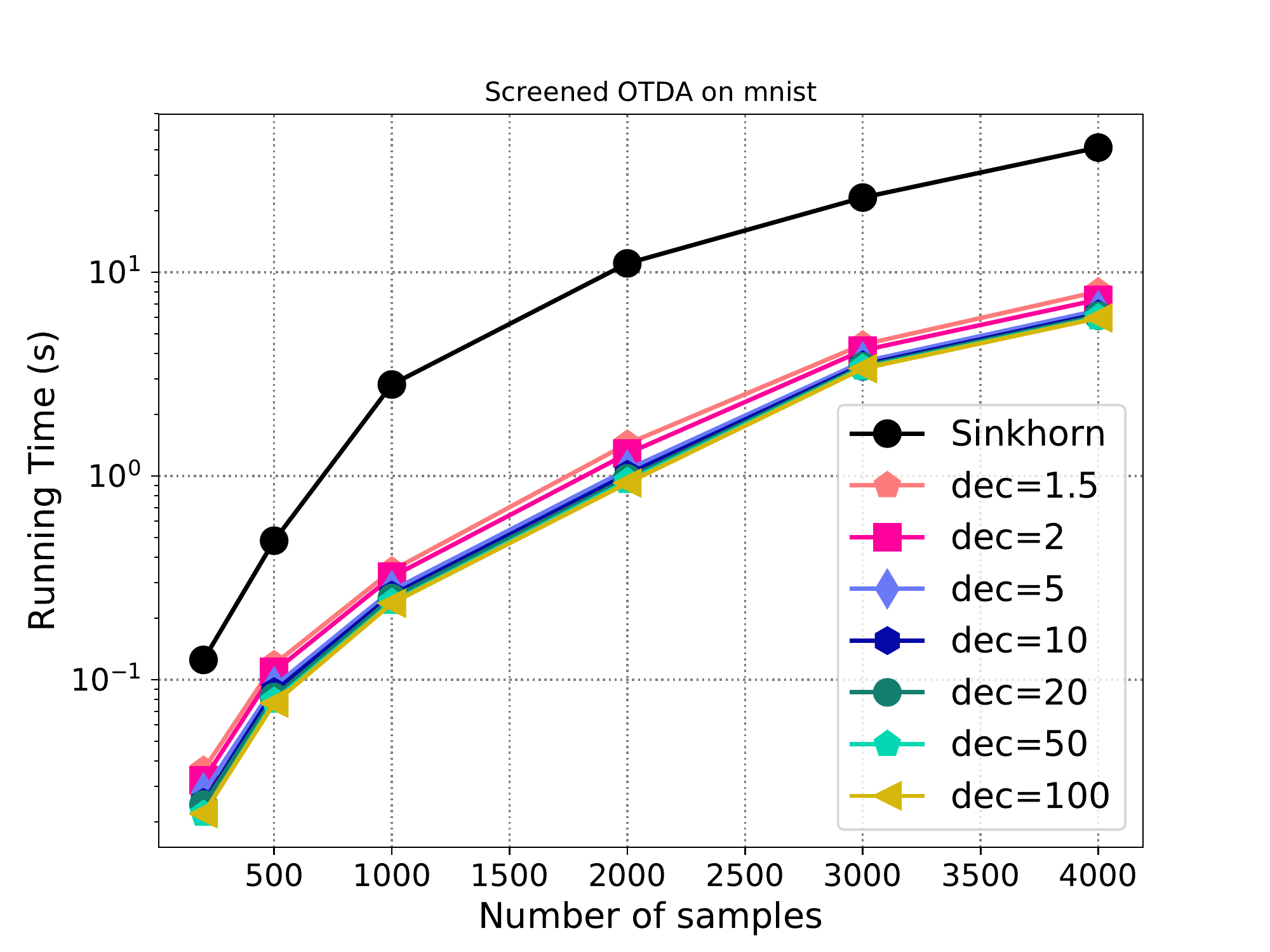}
	\caption{OT Domain adaptation MNIST to USPS : (top) Accuracy and (bottom) running time of \textsc{Sinkhorn} and \textsc{Screenkhorn} for hyperparameter of the $\ell_{p,1}$ regularizer (left) $\lambda = 1$ and (right) $\lambda=10$. Note that this
	value impacts the ground cost of each Sinkhorn problem involved in the iterative algorithm. The accuracy panels
	also report the performance of a $1$-NN when no-adaptation is performed.  
	We remark that the strenght of the class-based regularization has influence on the performance of \textsc{Screenkhorn} given a decimation factor. For small value on the left, \textsc{Screenkhorn} slightly performs better than \textsc{Sinkhorn}, while for large value, some
	decimation factors leads to loss of performances.
	Regarding, running time, we can note that \textsc{Sinkhorn} is far less efficient than
	\textsc{Screenkhorn}  with an order of magnitude for intermediate number of samples.
	\label{fig:otda:mnist:extra}}
\end{figure*}

\begin{figure*}[htbp]
	\centering

	\includegraphics[width=6.5cm]{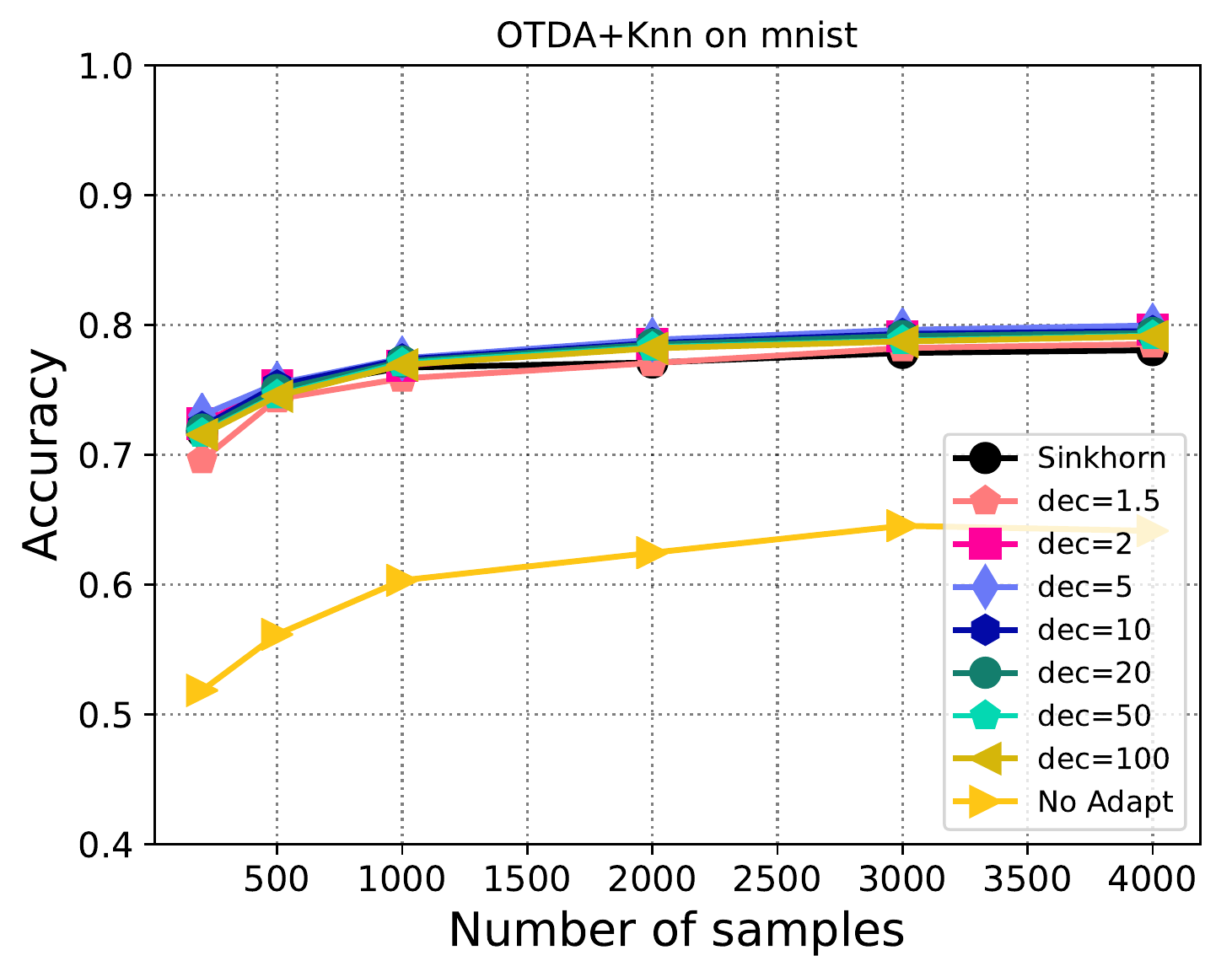}
	\includegraphics[width=6.7cm]{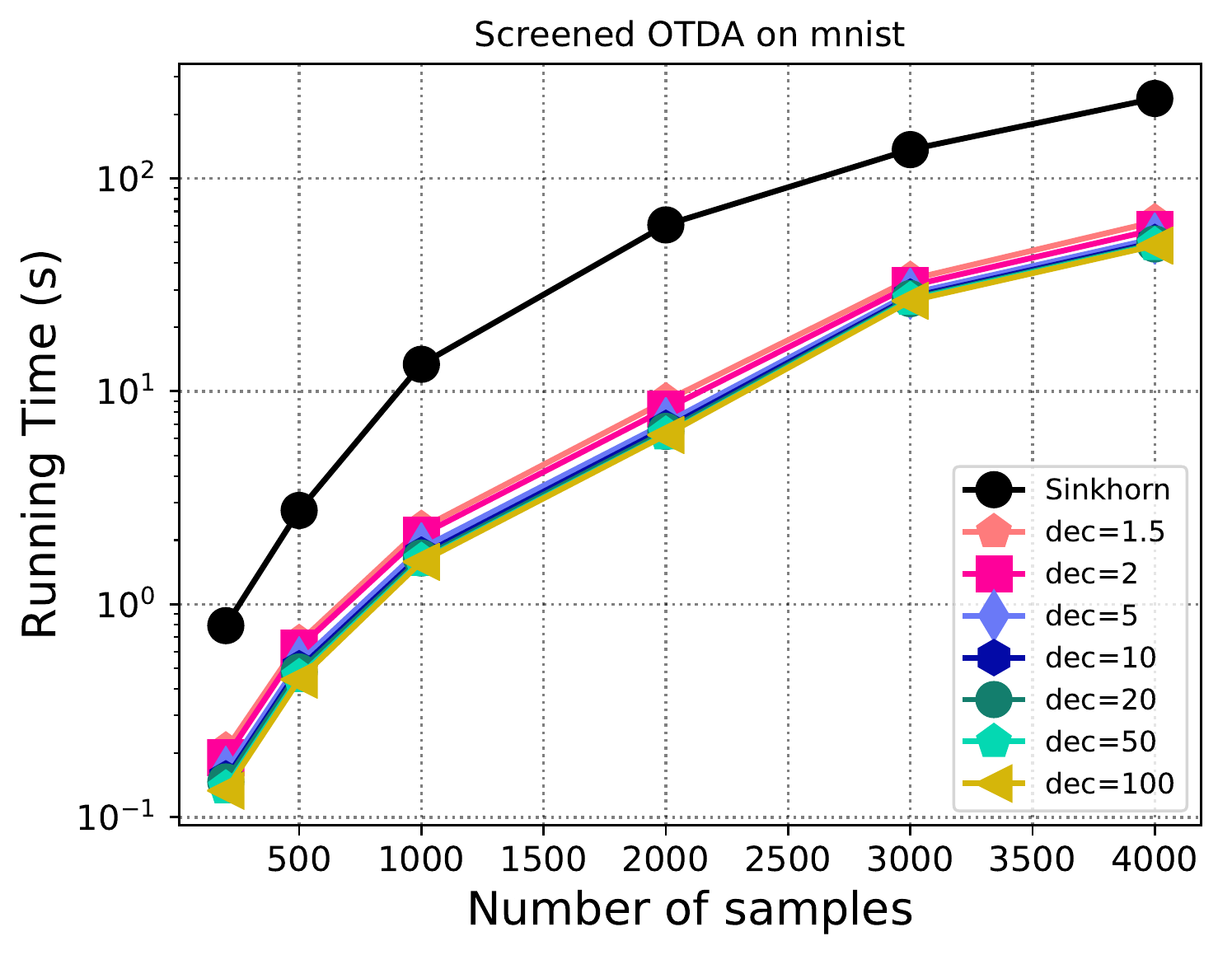}

	\caption{OT Domain adaptation MNIST to USPS : (left) Accuracy and (right) running time of \textsc{Sinkhorn} and \textsc{Screenkhorn} for the best performing (on average of $10$ trials) hyperparameter $\ell_{p,1}$ chosen among the set $\{0.1, 1,5,10 \}$. We can note that in this situation, there is not loss of accuracy while our \textsc{Screenkhorn} is still about an order of magnitude more efficient than Sinkhorn.}
\end{figure*}

\end{document}